\documentclass{article}

    \PassOptionsToPackage{numbers}{natbib}
\usepackage[final]{neurips2021}

\usepackage{xr-hyper}
\usepackage[utf8]{inputenc} 
\usepackage[T1]{fontenc}    
\usepackage{hyperref}       
\usepackage{url}            
\usepackage{booktabs}       
\usepackage{amsfonts}       
\usepackage{nicefrac}       
\usepackage{microtype}      
\usepackage{xcolor}         
\usepackage{tikz}
\usepackage{graphicx}
\usepackage{fancyhdr}
\usepackage{extramarks}
\usepackage{amssymb} 
\usepackage{datetime}
\usepackage{amsmath}
\usepackage{algorithm}
\usepackage{algorithmic}
\usepackage{amsthm}
\usepackage{enumerate}
\usepackage{subfigure}
\usepackage{wrapfig}
\usepackage{enumitem}

\usepackage{subfiles}

\makeatletter
\newcommand*{\addFileDependency}[1]{
  \typeout{(#1)}
  \@addtofilelist{#1}
  \IfFileExists{#1}{}{\typeout{No file #1.}}
}
\makeatother

\hypersetup{
    colorlinks=true,
    linkcolor=blue,
    filecolor=magenta,      
    urlcolor=blue,
}
\usetikzlibrary{positioning}
\usetikzlibrary{bayesnet}

\DeclareMathOperator*{\argmax}{arg\,max}
\DeclareMathOperator*{\argmin}{arg\,min}

\newtheorem{theorem}{Theorem}

\newtheorem{definition}{Definition}
\newtheorem{lemma}{Lemma}
\newtheorem{assumption}{Assumption}

\title{Regret Minimization Experience Replay in Off-Policy Reinforcement Learning}
\author{%
   Xu-Hui Liu\thanks{Equal contribution}~,\ Zhenghai Xue\footnotemark[1]~, Jing-Cheng Pang, Shengyi Jiang, Feng Xu,\ Yang Yu\thanks{Corresponding author} \\
  National Key Laboratory of Novel Software Technology \\
  Nanjing University, Nanjing 210023, China \\
  \texttt{liuxh@lamda.nju.edu.cn, xuezh@smail.nju.edu.cn} \\
  \texttt{\{pangjc, jiangsy, xufeng\}@lamda.nju.edu.cn, yuy@nju.edu.cn}
}
\hypersetup{
    colorlinks=true,
    linkcolor=blue,
    filecolor=magenta,      
    urlcolor=blue,
}
\begin{document}
\maketitle

\begin{abstract}
In reinforcement learning, experience replay stores past samples for further reuse. Prioritized sampling is a promising technique to better utilize these samples. Previous criteria of prioritization include TD error, recentness and corrective feedback, which are mostly heuristically designed. In this work, we start from the regret minimization objective, and obtain an optimal prioritization strategy for Bellman update that can directly maximize the return of the policy. The theory suggests that data with higher hindsight TD error, better on-policiness and more accurate Q value should be assigned with higher weights during sampling. Thus most previous criteria only consider this strategy partially. We not only provide theoretical justifications for previous criteria, but also propose two new methods to compute the prioritization weight, namely ReMERN and ReMERT. ReMERN learns an error network, while ReMERT exploits the temporal ordering of states. Both methods outperform previous prioritized sampling algorithms in challenging RL benchmarks, including MuJoCo, Atari and Meta-World.
\end{abstract}

\section{Introduction}\label{introduction}
Reinforcement learning (RL)~\cite{sutton} has achieved great success in sequential decision making problems. Off-policy RL algorithms~\cite{sac, td3, rainbow, dqn, c51} have the ability to learn from a more general data distribution than on-policy counterparts, and often enjoy better sample efficiency. This is critical when the data collection process is expensive or dangerous. Experience Replay ~\cite{er} enables data reuse and has been widely used in off-policy reinforcement learning. 
Previous work~\cite{diff} points out that emphasizing on important samples in the replay buffer can benefit off-policy RL algorithms. Prioritized Experience Replay (PER)~\cite{per} quantifies such importance by the magnitude of temporal-difference (TD) error. Based on PER, many sampling strategies ~\cite{sequence, discor, gan} are proposed to perform prioritized sampling. They are either based on TD error~\cite{per, sequence, gan} or focused on the existence of corrective feedback~\cite{discor}. However, these are all proxy objectives and different from the objective of RL, i.e., minimizing policy regret. They can be suboptimal in some cases due to this objective mismatch.

In this paper, we first give examples to illustrate the objective mismatch in previous prioritization strategies. Experiments show that lower TD error or more accurate Q function can not guarantee better policy performance. To tackle this issue, we first formulate an optimization problem that directly minimizes the regret of the current policy with respect to prioritization weights. We then make several approximations and solve this optimization problem. An optimal prioritization strategy is obtained and indicates that we should pay more attention to experiences with higher hindsight TD error, better on-policiness and more accurate Q value.  To the best of our knowledge, this paper is the first to optimize the sampling distribution of replay buffer theoretically from the perspective of regret minimization. 

We then provide tractable approximations to the theoretical results. The on-policiness can be estimated by training a classifier to distinguish recent transitions, which are generally more on-policy, from early ones, which are generally more off-policy. The oracle Q value is inaccessible during training, so we can not calculate the accuracy of Q value directly. Inspired by DisCor ~\cite{discor}, we propose an algorithm named ReMERN which estimates the suboptimality of Q value with an error network updated by Approximate Dynamic Programming (ADP).

ReMERN outperforms previous methods in environments with high randomness, e.g. with stochastic target positions or noisy rewards. However, the training of an extra neural network can be slow and unstable. We propose another estimation of Q accuracy based on a temporal viewpoint. With Bellman updates, the error in Q value accumulates from the next state to the previous one all across the trajectory. 
The terminal state has no bootstrapping target and low Bellman error. Therefore, states fewer steps away from the terminal state will have lower error in the updated Q value because of the more accurate Bellman target. This intuition is verified both empirically and theoretically. We then propose Temporal Correctness Estimation (TCE) based on the distance of each state to a terminal state, and name the overall algorithm ReMERT.

Similar to PER, ReMERN and ReMERT can be a plug-in module to all off-policy RL algorithms with a replay buffer, including but not limited to DQN~\cite{dqn} and SAC~\cite{sac}. Experiments show that ReMERN and ReMERT substantially improve the performance of standard  off-policy RL methods in various benchmarks.

\section{Background}

\subsection{Preliminaries}
A Markov decision process (MDP) is denoted $(\mathcal S, \mathcal A, T, r, \gamma, \rho_0)$, where $\mathcal S$  is the state space, and $\mathcal A$ is the action space. $T(s'|s, a)$ and $r(s, a)\in [0, \textnormal{R}_{\max}]$ are the transition and reward function. $\gamma\in (0, 1)$ is the discounted factor and $\rho_0(s)$ is the distribution of the initial state. The target of reinforcement learning is to find a policy that maximizes the expected return:
$\eta(\pi)=\mathbb E_\pi[\sum_{t\geq 0}\gamma^t r(s_t,a_t)]$, where the expectation is calculated from trajectories sampled from $s_0\sim \rho_0$, $a_t\sim \pi(\cdot|s_t)$, and $s_{t+1}\sim T(\cdot|s_t,a_t)$ for $t\geq 0$.

For a fixed policy, an MDP becomes a Markov chain, where the discounted stationary state distribution is defined as $d^\pi(s)$. With a slight abuse of notation, the discounted stationary state-action distribution is defined as $d^\pi(s,a)=d^\pi(s)\pi(a|s)$. Then the expected return can be rewritten as $\eta(\pi)=\frac{1}{1-\gamma}\mathbb E_{d^\pi(s,a)}[r(s, a)]$. We assume there exists an optimal policy $\pi^*$ such that $\pi^*=\argmax_\pi \eta(\pi)$. We use the standard definition of the state-action value function, or Q function: $Q^\pi(s, a)=\mathbb E_\pi[\sum_{t\geq 0}\gamma^t r(s_t, a_t)|s_0=s, a_0=a].$
 Let $Q^*$ be the shorthand for $Q^{\pi^*}$. $Q^*$ satisfies the Bellman equation $Q^*(s,a)=\mathcal B^*(Q^*(s,a))$, where $\mathcal{B}^* : \mathbb R^{\mathcal{S}\times\mathcal{A}} \to \mathbb R^{\mathcal{S}\times\mathcal{A}}$ is the Bellman optimal operator:
$(B^* f)(s, a):=r(s, a)+\gamma \max_{a'}\mathbb E_{s'\sim P(s, a)}f(s', a')$, where $ f\in \mathbb R^{\mathcal{S}\times\mathcal{A}}.$ 

The regret of policy $\pi$ is defined as $\textnormal{Regret}(\pi)=\eta(\pi^*)-\eta(\pi).$ It measures the expected loss in return by following policy $\pi$ instead of the optimal policy. Since $\eta(\pi^*)$ is a constant, minimizing the regret is equivalent to maximizing the expected return, and thus it can be an alternative objective of reinforcement learning.

\subsection{Related Work}
Extensive researches have been conducted on experience replay and replay buffer. The most frequently considered aspect is the sampling strategy. Various techniques have achieved good performance by performing prioritized sampling on the replay buffer. In model-based planning, Prioritized Sweeping~\cite{sweep1, sweep2, sweep3} selects the next state updates according to changes in value. Prioritized Experience Replay (PER)~\cite{per} prioritizes samples with high TD error. Taking PER one step further, Prioritized Sequence Experience Replay (PSER)~\cite{sequence} considers information provided by transitions when estimating TD error. Emphasizing Recent Experience (ERE)~\cite{ere} and Likelihood-Free Importance Weighting (LFIW)~\cite{gan} prioritizes the correction of TD errors for frequently encountered states. Distribution Correction (DisCor)~\cite{discor} assigns higher weights to samples with more accurate target Q value because these samples provide ``corrective feedback''. DisCor uses a neural network to estimate the accuracy of target Q value. Inspired by DisCor, SUNRISE~\cite{lee2020sunrise} proposes to use the variance of ensembled Q functions as a surrogate for the accuracy of Q value. Adversarial Feature Matching ~\cite{bottleneck}  focuses on sampling uniformly among state-action pairs.

Instead of proposing a new strategy of sampling,~\cite{equivalence} proves that there exists a relationship between sampling strategy and loss function, and weighted value loss can serve as a surrogate for prioritized sampling. Other works focus on buffer capacity~\cite{deeper, revisit}. They point out that a proper buffer capacity can accelerate value estimation and lead to better learning efficiency and performance. In fact, this can be thought of as a specific example of prioritization strategies, i.e., assigning zero weights to the samples exceeding the proper buffer capacity.

\section{Optimal Prioritization Strategy via Regret Minimization}\label{theory}
\subsection{Revisiting Existing Prioritization Methods}
\label{motivation}
PER and DisCor are two representative algorithms of prioritized sampling. PER prioritizes state-action pairs with high TD error, while DisCor prefers to perform Bellman update on state-action pairs that have more accurate Bellman targets. However, both criteria are different from the target of RL algorithms, which is to maximize the expected return of the policy. Such difference can slow down the training process in some cases. For example, when the Bellman target is inaccurate, minimizing TD error does not necessarily improve the optimality of Q value.
\usetikzlibrary{arrows.meta,bending,positioning,decorations.text}
\tikzset{
leftvecArrow/.style={
  draw, black, {Triangle[bend, angle=40:2pt 3]}-, shorten >=2pt, shorten <=2pt, postaction=decorate, decoration={text along path, text color=black, text align=center, text={|\scriptsize|#1 ||}}
},
rightvecArrow/.style={
  draw, black, -{Triangle[bend, angle=40:2pt 3]}, shorten >=2pt, shorten <=2pt, postaction=decorate, decoration={text along path, text color=black, text align=center,  text={|\scriptsize|#1 ||}}
}
}

\begin{figure}
\label{fig:mdp}
\centering
\subfigure[]{\label{fig:mdp_ex}
    \begin{tikzpicture}[label distance = 5mm]
    \tikzstyle{every node}=[scale=0.8]
      \node[circle, draw, fill=black!50, line width=0.3mm] (a) {$s_T$};
      \node[circle, draw, line width=0.25mm, node distance=0.5cm, right=of a] (b) {$s_0$};
      \node[circle, draw, line width=0.3mm, , node distance=0.5cm, right=of b] (c) {$s_1$};
      \node[circle, draw, line width=0.3mm, node distance=0.5cm, right=of c] (d) {$s_2$};
      \node[circle, draw, line width=0.3mm, node distance=0.5cm, right=of d] (e) {$s_3$};
      \path [leftvecArrow=+2] (a.north east) to[in=165,out=20] (c.north);
      \path [leftvecArrow=+2] (a.north) to[in=165,out=20] (e.north);
      \path [leftvecArrow=+2] (a.south east) to[in=200,out=345] (b.south);
      \path [leftvecArrow=+2] (a.south) to[in=200,out=345] (d.south);

      \path [rightvecArrow=+1] (b.east) to[in=180,out=0] (c.west);
      \path [rightvecArrow=+1] (c.east) to[in=180,out=0] (d.west);
      \path [rightvecArrow=+1] (d.east) to[in=180,out=0] (e.west);
    \end{tikzpicture}
    }
\subfigure[]{
    \label{fig:mdp_per}
    \includegraphics[width=0.27\textwidth]{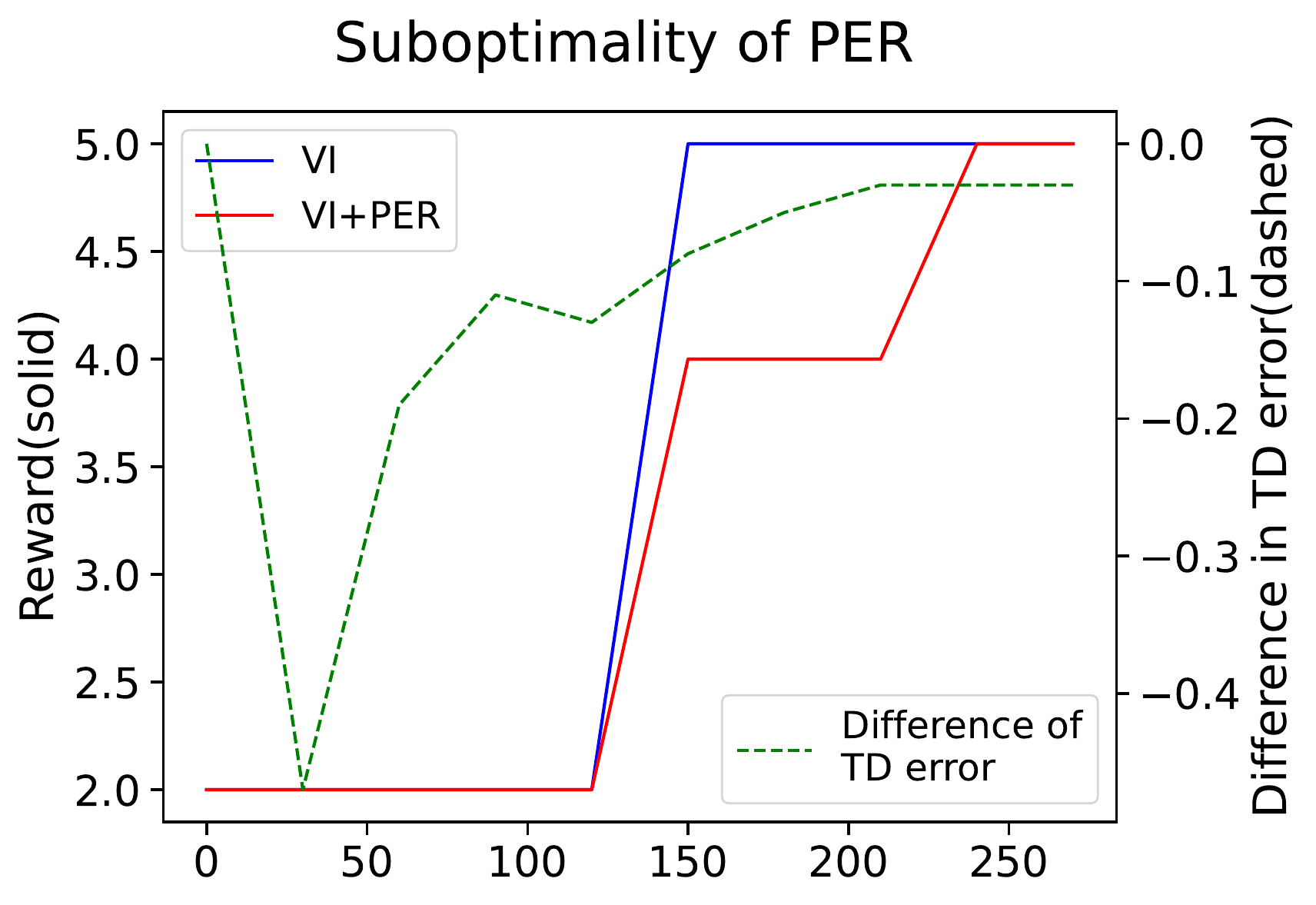}
    }\subfigure[]{
    \label{fig:mdp_discor}
    \includegraphics[width=0.28\textwidth]{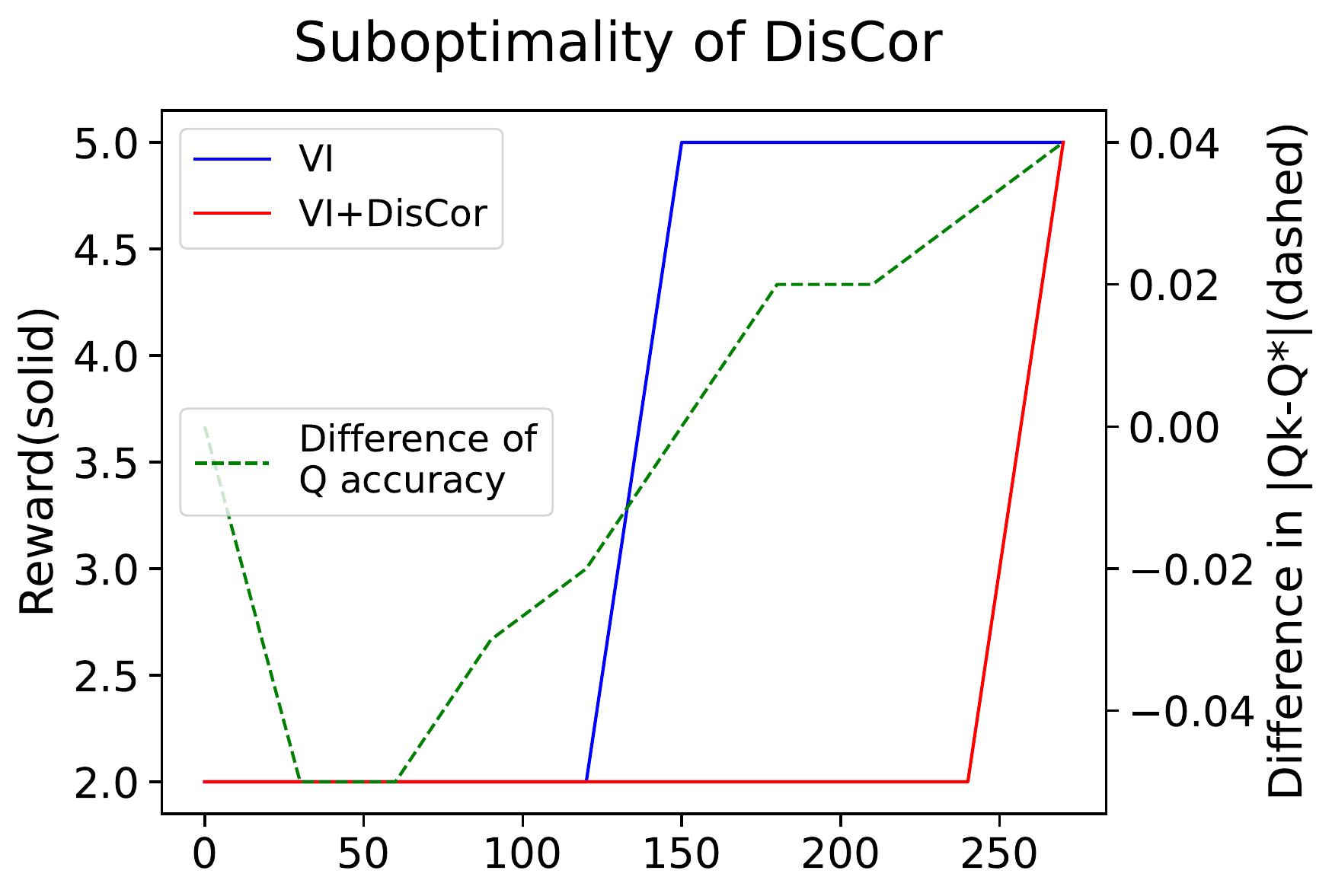}
    }

\caption{A simple MDP showing the objectives of PER and DisCor can slow down the training process. (a)  A 5-state MDP with initial state $s_0$ and terminal state $s_T$. Except for $s_3$ and $s_T$, there are two available actions, \textit{left} and \textit{right}. Turning \textit{left} leads to the terminal state $s_T$ and +2 reward, while turning \textit{right} leads to the next state and +1 reward. The optimal policy is to keep turning \textit{right} until reaching $s_3$, then reach $s_T$. (b) Relationship between TD error (dashed line) and performance (solid line) of VI and VI+PER. (c) Relationship between  Q error (dashed line) and performance (solid line) of VI and VI+DisCor. }
\label{fig:mdp}
\end{figure}
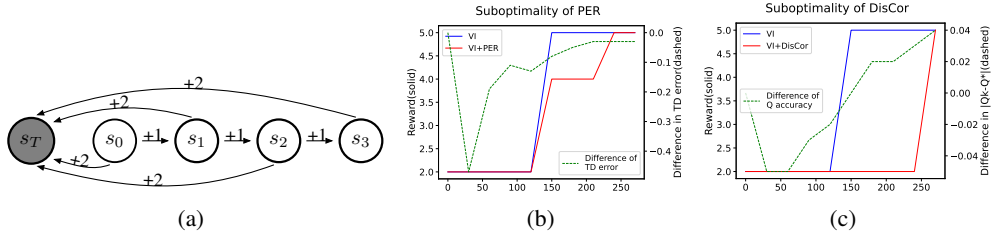

To illustrate the aforementioned problems, we provide an example MDP shown in the left part of Fig.~\ref{fig:mdp_ex}. This is a five-state MDP with two actions: turning left and right. The optimal policy is to turn \textit{right} in all states, receiving a total reward of $5$. Suppose the Q values for all $(s,a)$ pairs are initialized to zero. The reward of turning \textit{left} is higher than turning \textit{right} in all states, so the \textit{left} action has a higher TD error. As a result, PER prefers states with the \textit{left} action, which is not the fastest training process to achieve the optimal policy. 
Also, since there is no bootstrapping error for the terminal state $s_T$, transitions with $s_T$ as the next state have an accurate target Q value. Therefore, DisCor also focuses on state-actions pairs with \textit{left} action, which is again not optimal. 

We perform Value Iteration (VI) on this MDP. To simulate function approximation in Deep RL and avoid convergence in few iterations, the learning rate is set to $0.1$. Prioritized sampling is substituted by weighted Bellman update, as introduced in~\cite{equivalence}. The results are shown in Fig.~\ref{fig:mdp_per} and Fig.~\ref{fig:mdp_discor}. According to the results, PER indeed minimizes TD error more efficiently, and DisCor results in a more accurate estimation of Q value, as indicated  by their objectives. However, they both need more iterations to converge than value iteration without prioritization. According to this MDP, the objective of previous prioritization methods can be inefficient in certain cases. 

\subsection{Problem Formulation of Regret Minimization}
As shown in Section~\ref{motivation}, an indirect objective can cause slower convergence of value iteration. In this section, we aim to find an optimal prioritization weight $w_k$ that can directly minimize the policy regret $\eta(\pi^*) - \eta(\pi_{k})$. The weight is multiplied to the Bellman error $(Q-\mathcal{B}^*Q_{k-1})^2$ and $\pi_k$ can be obtained from the updated Q function. To facilitate further derivations, we only consider the best Q function of the Bellman update, which is calculated by the $\argmin$ operator. Therefore, the optimization problem with respect to $w_k$ can be written as:
\begin{equation}\label{eq_regret}
\begin{aligned}
    &\min_{w_k} \qquad \qquad \eta(\pi^*)-\eta(\pi_{k})\\
    &~\textnormal{s.t.}\quad Q_k=\argmin_{Q\in\mathcal{Q}}~\mathbb E_{\mu}[w_k(s,a)\cdot(Q-\mathcal B^* Q_{k-1})^2(s,a)],\\
    &\qquad \ \ \mathbb E_\mu[w_k(s,a)]=1, \quad w_k(s,a)\geq 0,
\end{aligned}
\end{equation}
where $\pi_{k}(s)=\frac{\exp(Q_k(s,a))}{\sum_{a'}\exp(Q_k(s,a'))}$ is the policy corresponding to $Q_k$. $\mathcal{Q}$ is the function space of Q functions and $\mu$ is the data distribution of the replay buffer. $Q_k$ is the estimate of Q value after the Bellman update at iteration $k$.



We then manage to solve this optimization problem. To get started, we introduce \textit{recurring probability} which serves as an upper bound of the error term in our solution.

\begin{definition}[Recurring Probability]
\label{def_recurring}
The recurring probability of a policy $\pi$ is defined as $\epsilon_\pi = \sup_{s,a}\sum_{t=1}^\infty \gamma^t\rho^{\pi}(s,a,t)$, where $\rho$ is the probability of the agent starting from $(s,a)$ and coming back to $s$ at time step $t$ under policy $\pi$, i.e.,  $\rho^{\pi}(s, a, t)=\textnormal{Pr}(s_0=s, a_0=a, s_t=s, s_{1:t-1}\neq s; \pi)$.
\end{definition}

We then present the solution to the optimization problem \ref{eq_regret} in Thm.~\ref{thm_optimization}. The formal version of the theorem and detailed proof are in Appendix \ref{proof}.
\begin{theorem}[Informal]\label{thm_optimization}
Under mild conditions, the solution $w_k$ to a relaxation of the optimization problem \ref{eq_regret} in MDPs with discrete action spaces is
\begin{equation}\label{eq_discrete_theorem}
    w_k(s,a )=\frac{1}{Z_1^*}\left( E_k(s,a)+\epsilon_{k,1}(s,a)\right). 
\end{equation}
In MDPs with continuous action spaces, the solution is
\begin{equation}\label{eq_continuous_theorem}
    w_k(s,a)=\frac{1}{Z_2^*}\left( F_k(s,a)+\epsilon_{k,2}(s,a)\right).  
\end{equation}
where $$E_k(s,a)=\underbrace{\frac{d^{\pi_{k}}(s,a)}{\mu(s,a)}}_{(a)}\underbrace{(2-\pi_{k}(a|s))}_{(b)}\underbrace{\exp \left(-\left|Q_{k}-Q^{*}\right|(s, a)\right)}_{(c)}\underbrace{\left|Q_{k}-\mathcal{B}^* Q_{k-1}\right|(s, a)}_{(d)}$$
$$F_k(s,a)=2\underbrace{\frac{d^{\pi_{k}}(s,a)}{\mu(s,a)}}_{(a)}\underbrace{\exp \left(-\left|Q_{k}-Q^{*}\right|(s, a)\right)}_{(c)}\underbrace{\left|Q_{k}-\mathcal{B}^* Q_{k-1}\right|(s, a)}_{(d)},$$
$Z_1^*$, $Z_2^*$ are normalization factors, $\epsilon_{k,1}(s,a)$ and $\epsilon_{k,2}(s,a)$ satisfy $\max\left\{\frac{\epsilon_{k,1}(s,a)}{E_k(s,a)}, \frac{\epsilon_{k,2}(s,a)}{F_k(s,a)}\right\}\leq \epsilon_{\pi_{k}}$.
\end{theorem}

With regard to the error terms,  there are two cases where $\epsilon_{\pi_{k}}$ is low by its definition: the probability of coming back to the states that have been visited is small, or the number of steps an agent takes to come back to the visited states is large. In most tasks, either of these cases holds. We conduct experiments in several Atari games and show the  verification results in Appendix \ref{experiments}. The low probability leads to small $\epsilon_{\pi_{k}}$ and implies the terms $\epsilon_{k,1}(s,a)$ and $\epsilon_{k,2}(s,a)$ are negligible. 

Therefore, Thm.~\ref{thm_optimization} suggests that state-action tuples in the replay buffer should be assigned with higher importance if they have the following properties:
\begin{itemize}[leftmargin=*]
    \item \textbf{Higher hindsight Bellman error} ( from $|Q_k - \mathcal{B}^* Q_{k-1}|(s,a)$). $Q_k$ is the estimate of Q value after the Bellman update. This term describes the difference between the estimated hindsight Q value and the Bellman target. It is similar to the prioritization criterion of PER \cite{per}, but PER concerns more about the historical Bellman error, i.e., $|Q_{k-1}-\mathcal{B}^* Q_{k-2}|(s,a)$. 
    \item \textbf{More on-policiness} ( from $\frac{d^{\pi_{k}}(s,a)}{\mu(s,a)}$). 
    An efficient update of $\pi$ requires $w_k$ to be on-policy, i.e., focusing on state-action pairs which are more likely to be visited by the current policy. Such prioritization strategy has been empirically illustrated in LFIW~\cite{gan} and BCQ~\cite{bcq}, while we obtain it directly from our theorem.
    \item \textbf{Closer value estimation to oracle} ( from $\exp \left(-\left|Q_{k}-Q^{*}\right|(s, a)\right)$ ). This term indicates that state-action pairs with less accurate Q values after the Bellman update should be assigned with lower weights. Intuitively, state-action pairs that lead to suboptimal updates of the estimator of Q value should be down-weighted. Such suboptimality may arise from incorrect target Q values or the error of function approximation in deep Q networks. 
    \item \textbf{Smaller action likelihood} (from $2 - \pi_{k}(a|s)$). This term only exists in MDPs with a discrete action space. It offsets the effect of the on-policy term $d^{\pi_{k}}$ to some extent and is similar to $\varepsilon$-greedy strategy in exploration.
\end{itemize}

Our theoretical analysis indicates that existing prioritization strategies only consider the problem partially, neglecting other terms in minimizing the regret. For example, DisCor fails to consider the on-policiness and PER ignores the accuracy of Q value.
In the remaining part of this section, we present practical approximations to each term in Eq.~(\ref{eq_discrete_theorem}) and (\ref{eq_continuous_theorem}).

Term (a) is the importance weight between the current policy and the behavior policy. We can calculate this term using Likelihood-Free Importance Weighting (LFIW,~\cite{gan}). LFIW divides the replay buffer into two parts, a fast buffer $\mathcal D_f$ and a slow buffer $\mathcal D_s$. It initializes a neural network $\kappa_\psi(s,a)$ and optimizes the network according to the following loss function:
\begin{equation}\label{eq_lfiw}
L_{\kappa}(\psi):=\mathbb{E}_{\mathcal{D}_{\mathrm{s}}}\left[f^{*}\left(f^{\prime}\left(\kappa_{\psi}(s, a)\right)\right)\right]-\mathbb{E}_{\mathcal{D}_{\mathrm{f}}}\left[f^{\prime}\left(\kappa_{\psi}(s, a)\right)\right],
\end{equation}
where $f'$ and $f^*$ is the derivative and convex conjugate of function $f$. The updated $\kappa_\psi$ is the desired importance weight.

For term (b) and (d), since $\pi_{k}$ and $Q_k$ are the policy and the estimate of Q value after the update, they cannot be calculated directly. Therefore, we approximate them by the upper and lower bounds. For term (b), $1\leq 2-\pi_{k}(a|s)\leq 2$. For term (d), a viable approximation is to bound it between the minimum and maximum Bellman errors obtained at the previous iteration, $c_1=\min_{s,a}|Q_{k-1}-\mathcal B^*Q_{k-2}|$ and $c_2=\max_{s,a}|Q_{k-1}-\mathcal B^*Q_{k-2}|$. As shown in DisCor, we can restrict the support of state-action pairs $(s, a)$ used to compute $c_1$ and $c_2$ in the support of replay buffer, to ensure that both $c_1$ and $c_2$ are finite. With these approximation, we can derive a lower bound for $w_k$, which will be detailed in Sec.~\ref{sec:remern} and Sec.~\ref{temporal}.

In the next two subsections, we will provide two practical algorithms to estimate $|Q_k-Q^*|$.

\subsection{Regret Minimization Experience Replay Using Neural Network (ReMERN)}
\label{sec:remern}
DisCor shows $\Delta_k$ can be a surrogate of $|Q_k-Q^*|$, which is defined as:
\begin{align}
\Delta_{k} &=\sum_{i=1}^{k} \gamma^{k-i}\left(\prod_{j=i}^{k-1} P^{\pi_{j}}\right)\left|Q_{i}-\mathcal{B}^{*} Q_{i-1}\right| \\
\Longrightarrow \Delta_{k} &=\left|Q_{k}-\mathcal{B}^{*} Q_{k-1}\right|+\gamma P^{\pi_{k-1}} \Delta_{k-1} \label{eq_delta}
\end{align}

According to Eq.~(\ref{eq_delta}), $\gamma [P^{\pi_{k-1}}\Delta_{k-1}](s, a) + c_2$ is an upper bound of $|Q_k-Q^*|$. This is because $\Delta_k$ is proven to be the upper bound of $|Q_k-Q^*|$ ~\cite{discor} and $c_2$ is the upper bound of $|Q_k-\mathcal{B}^*Q_{k-1}|$. Recall that $2-\pi_{k}(a|s)\geq 1$ and $|Q_{k-1}-\mathcal{B}^*Q_{k-2}|\geq c_1$, and we can derive the final expression for this tractable approximation for $w_k(s,a)$ by simplifying all constants: 
\begin{equation}\label{eq_discor_cwk}
    w_k(s, a)\propto \frac{d^{\pi_{k}}(s, a)}{\mu(s, a)}\exp \left(-\gamma\left[P^{\pi_{k-1}} \Delta_{k-1}\right](s, a)\right),
\end{equation}
This approximation applies to MDPs with discrete action space and MDPs with continuous action space. Using the lower bound of $w_k(s,a)$ may down-weight some transitions, but will never up-weight a transition by mistake~\cite{discor}.

We use a neural network to estimate $\Delta_{k-1}$. As shown in Eq.~(\ref{eq_delta}), $\Delta_{k-1}$ can be calculated from a bootstrapped target, which inspires us to use ADP algorithms to update it. 
We name this method ReMERN (\textbf{Re}gret \textbf{M}inimization \textbf{E}xperience \textbf{R}eplay using \textbf{N}eural Network). The pseudo code for ReMERN is presented in Appendix \ref{algorithms}. ReMERN is applicable to all value-based off-policy algorithms with replay buffer. 

\subsection{Regret Minimization Experience Replay Using Temporal Structure (ReMERT)}\label{temporal}
ReMERN uses neural network as the estimator of $|Q_k-Q^*|$. However, training a neural network is time consuming and suffers from large estimation error without adequate iterations. To mitigate this issue, we propose another estimation of $|Q_k-Q^*|$ from a different perspective. 
\subsubsection{The Temporal Property of Q Error}

$|Q_k-Q^*|$ can be decomposed with the triangle inequality: $|Q_k-Q^*|\leq |Q_k-\mathcal{B}^*Q_{k-1}|+|\mathcal{B}^*Q_{k-1}-Q^*|.$ 
The first term is the projection error depending on the Q function space $\mathcal{Q}$. This error is usually small thanks to the strong expressive power of neural networks. In the second term, $\mathcal{B}^*Q_{k-1}$ is the estimate of target Q value, and $|\mathcal{B}^*Q_{k-1}-Q^*|$ is the distance from the target Q value to the ground-truth Q value. The target Q value at the terminal state consists of the reward only, so there is no bootstrapping error and $|\mathcal{B}^*Q_{k-1}-Q^*|=0$. Moving backward through the trajectory, the accuracy of the Q value estimation decreases as the error of Bellman update accumulates. These Q values are then utilized to compute the target Q value, leading to more erroneous Bellman updates and larger $|\mathcal{B}^*Q_{k-1}-Q^*|$. Such error can accumulate through the MDP. Consequently, states closer to the terminal state tend to have a more accurate Bellman target. 
 This motivates us to estimate the incorrectness of the estimated Q value using the temporal information of a given state-action tuple $(s_t, a_t)$. 

\begin{wrapfigure}[12]{R}{0.4\textwidth}
\vspace{-1.7em}
\centering
\includegraphics[width=0.3\textwidth, trim=20 0 0 0, clip]{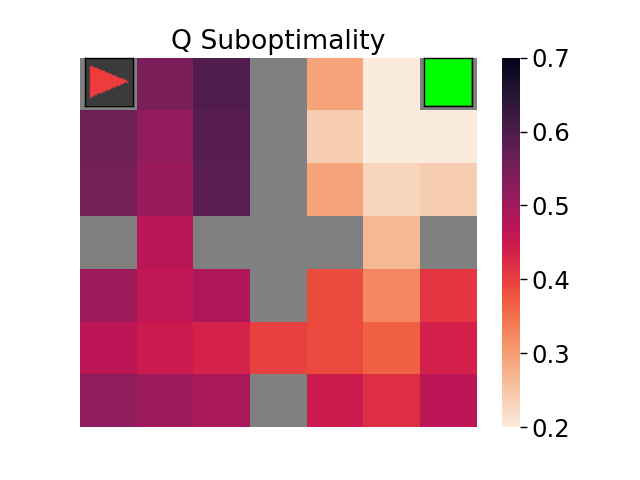}
\vspace{-1.5em}
\caption{The visualized error of target Q value in a GridWorld Environment. The Q error is visualized by the color of the grid. }
\label{fig:maze}
\end{wrapfigure}

To verify our intuition on the temporal property of Q error, we use a gridworld MDP from~\cite{gym_minigrid} and visualize the mean error of the target Q value (i.e., $|\mathcal{B}^*Q_{k-1}-Q^*|$) across different actions in Fig.~\ref{fig:maze}. We use DQN to update Q values. In this gridworld MDP, an agent starts at the red triangle on the top-left and terminates at the green rectangle on the top-right. The agent can't go through the wall, which is plotted as gray grids. The darker a grid is, the higher error of Q function it has. This figure illustrates that states closer to the terminal state has lower Q error, corresponding to our intuition that $|\mathcal{B}^*Q_{k-1}-Q^*|$ is related to the position of $(s,a)$ in the trajectory.

To formalize this intuition, we first define \textit{Distance to End}.

\begin{definition}[Distance to End]
Given a MDP $\mathcal{M}$, $\tau=\{s_t, a_t\}_{t=0}^T$ is a trajectory generated by policy $\pi$ in $\mathcal{M}$. The distance to end of $(s_t, a_t)$, denoted by $h^\pi_\tau(s_t, a_t)$, is $T-t$ in this trajectory. 
\end{definition}
Our intuition states that the value of $|Q_k-Q^*|$ has a positive correlation with distance to end. Based on this intuition, we propose the following theorem. 
\begin{theorem}[Informal]\label{thm_step}
Under mild conditions, with probability at least $1-\delta$, we have
\begin{equation}
\begin{aligned}
&\quad \ |Q_k(s,a)-Q^*(s,a)|\\
&\leq \mathbb E_{\tau}\bigg(f(h^{\pi_{k}}_\tau(s,a))\big(L_{Q_{k-1}}+c\big)+\gamma^{h^{\pi_{k}}_\tau(s,a)+1}c\bigg)+g(k, \delta)
\end{aligned}
\end{equation}
where $c=\max_{s, a}\big(Q^*(s,a^*)-Q^*(s,a)\big)$, $f(t)=\frac{\gamma-\gamma^{t}}{1-\gamma}$, $L_{Q_{k-1}}=\mathbb E[|Q_{k-1}-\mathcal{B}^*Q_{k-2}|]$ and $g(k, \delta)$ decreases exponentially as $k$ increases.
\end{theorem}

The formal version of the theorem and its proof are in Appendix \ref{proof2}. The theorem states that $|Q_k-Q^*|$ is upper bounded by a function of distance to end and expected Bellman error with high probability.

\subsubsection{A Practical Implementation}
In Thm.~\ref{thm_step} we derive the upper bound of $|Q_k-Q^*|$, which can serve as a surrogate to $|Q_k-Q^*|$. Using an upper bound as the surrogate may down-weight some transitions, but will never up-weight a transition that should not be up-weighted~\cite{discor}. We call this Temporal Correctness Estimation (TCE):

\begin{equation}\label{eq_tau}
\begin{aligned}
    |Q_k(s,a)&-Q^*(s,a)|\approx \mathbb E_{\tau}\textnormal{TCE}_c(s,a) \\
    &=\mathbb E_{\tau}\bigg(f(h^{\pi_{k-1}}_\tau(s,a))\big(L_{Q_{k-1}}+c\big)+\gamma^{h^{\pi_{k-1}}_\tau(s,a)+1}c\bigg),
\end{aligned}
\end{equation}


Similar to the derivation of ReMERN, we can simplify the expression of $w_k(s,a)$ as: 
\begin{equation}\label{eq_new}
    w_k(s,a)\propto \frac{d^{\pi_{k}}(s, a)}{\mu(s, a)}\exp \Big(-\mathbb E_{\tau}\text {TCE}_c(s,a)\Big)
\end{equation}




This approach of computing prioritization weights is named ReMERT (\textbf{Re}gret \textbf{M}inimization \textbf{E}xperience \textbf{R}eplay using \textbf{T}emporal Structure). Its pseudo code is presented in Appendix \ref{algorithms}. In practice, we record the \textit{distance to end} of a state-action pair when it is sampled by the policy and stored in the replay buffer. The expectation with respect to $\tau$ is computed based on the record and Monte-Carlo estimation.

\subsection{Comparison between ReMERN and ReMERT}
ReMERT can estimate $|Q_k-Q^*|$ directly from the temporal ordering of states, which often provides more efficient and more accurate estimation than ReMERN. However, The expectation with respect to trajectory $\tau$ in Eq.~(\ref{eq_new}) induces statistical error. In some environments, the \textit{distance to end} of a certain state-action pair $(s,a)$ can vary widely across different trajectories, which is usually caused by the randomness of environments. For example, in environments with stochastic goal positions, the state may be near the goal in one episode but far away from it in another. In such cases, prioritization weights provided by ReMERT have large variance and can be misleading.
In contrast, ReMERN need to train an error net but is irrelevant to the \textit{distance to end}. Therefore, ReMERN suffers estimation error of neural network but is robust to the randomness of environments. We test their property in the following section.

\section{Experiments}
In this section, we conduct experiments to evaluate ReMERN and ReMERT\footnote{Codes are available at \href{https://github.com/AIDefender/ReMERN-ReMERT}{https://github.com/AIDefender/ReMERN-ReMERT}.}. We choose SAC and DQN as the baseline algorithms for continuous and discrete action space respectively and incorporate ReMERN and ReMERT as the sampling strategy.  We first compare the performance of ReMERN and ReMERT to prior sampling methods in continuous control benchmarks including Meta-World~\cite{metaworld}, MuJoCo~\cite{mujoco} and Deepmind Control Suite (DMC)~\cite{tassa2020dmcontrol}. We also evaluate our methods in Arcade Learning Environments with discrete action spaces. Then, we dive into our algorithms and design several experiments, such as Gridworld tasks and MuJoCo with reward noise, to demonstrate some key properties of ReMERN and ReMERT. A detailed description of the environments and experimental details are listed in Appendix \ref{experiments}.

\begin{figure}[bt!]
\centering
    \includegraphics[width=0.8\textwidth   ]{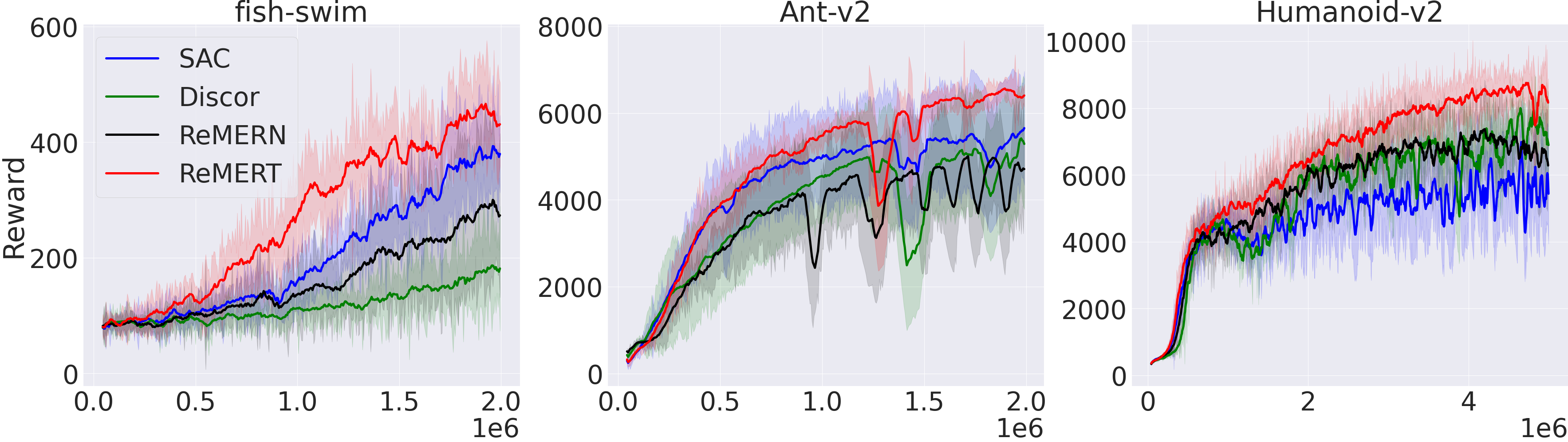}
    \includegraphics[width=0.8\textwidth]{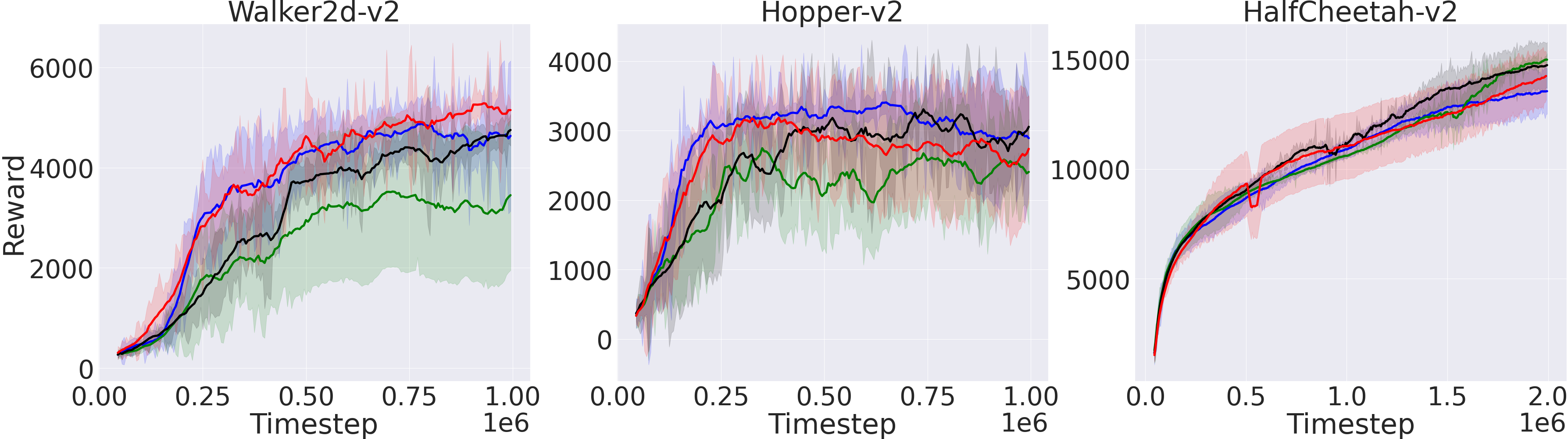}
    \caption{Performance of ReMERT, ReMERN with SAC and DisCor as baselines on continuous control tasks. }
    \label{fig:mujoco}
\end{figure}

\subsection{Performance on Continuous Control Environments}

In MuJoCo and DMC tasks, ReMERT outperforms baseline methods on four of six tasks and achieves comparable performance in the rest two tasks, i.e. HalfCheetah and Hopper, as shown in Fig.~\ref{fig:mujoco}. The marginal improvement of ReMERT in HalfCheetah mainly comes from the absence of a strong correlation between Q-loss and time step. In HalfCheetah, there is no specific terminal state, so the agent always reaches the max length of the trajectory, which gives a fake "distance to end" for every state. 
In Hopper, there is not much difference of the $|Q_k-Q^*|$ term between all the sampled state-action pairs, as shown in Appendix~\ref{experiments}, so the 
state-action pairs are not sampled very unequally. Besides, Hopper is a relatively easy task, in which prioritizing the samples have minor impact on the overall performance of the RL algorithm. The performance of ReMERN is better than DisCor, but is not as good as ReMERT. This verifies our theory and the existence of large estimation error induced by updating neural network with ADP algorithms. 

The Meta-World benchmark~\cite{metaworld} includes many robotic manipulation tasks. We select 8 tasks for evaluation, and plot the result in Fig.~\ref{fig:metaworld}. The performance of PER can be found in its paper~\cite{per}. Current state-of-the-art off policy RL algorithms such as SAC performs poorly on this benchmark because the goals of tasks have high randomness. Although DisCor~\cite{discor} shows preferable performance in these tasks compared to SAC and PER, ReMERN obtains a significant improvement over DisCor in the training speed and asymptotic performance.  
In this evaluation, we exclude ReMERT for comparison because the randomized target position in Meta-World contradicts its assumption.



\begin{figure}[htb]
    \includegraphics[width=\textwidth]{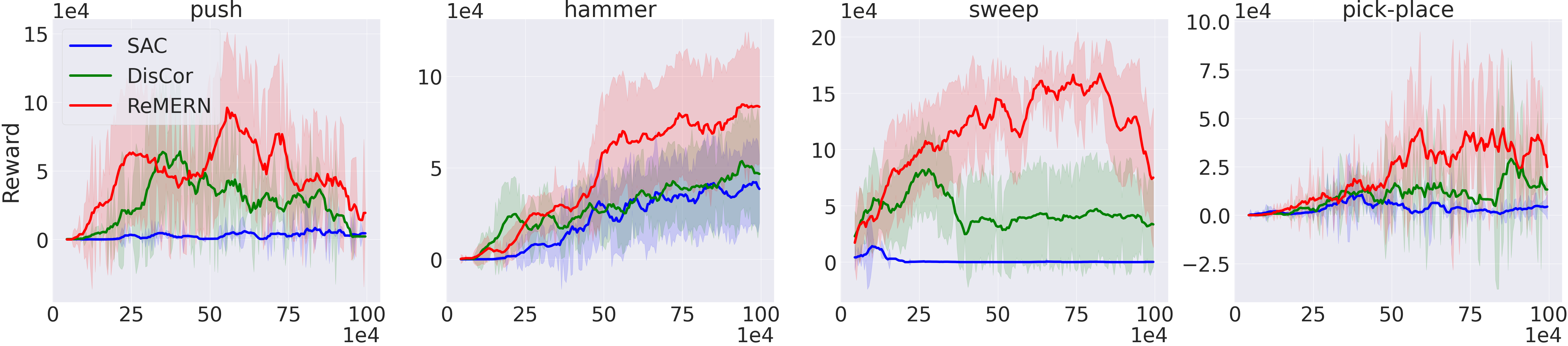}
    \includegraphics[width=\textwidth]{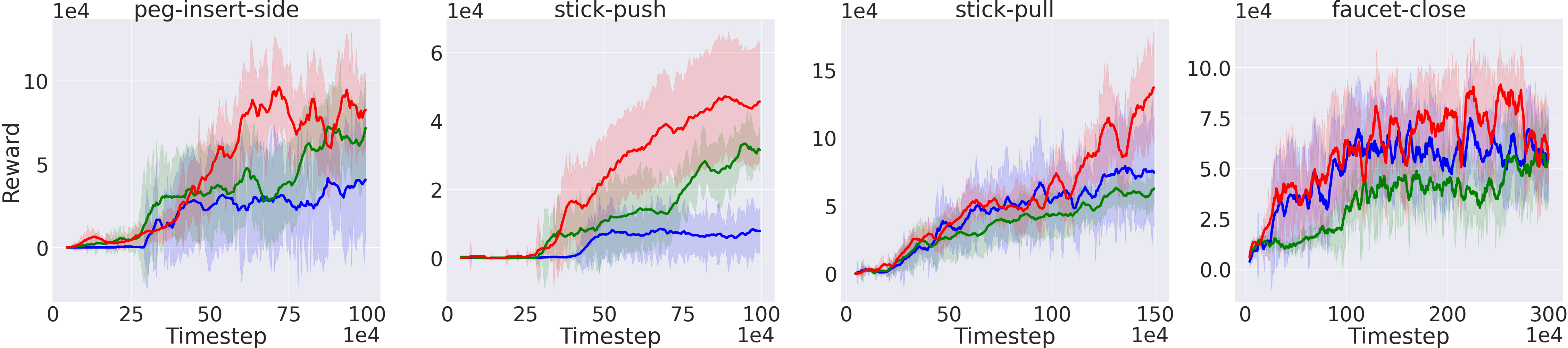}
    \vspace{-3mm}
    \caption{Performance of ReMERN, standard SAC and DisCor in eight Meta-World tasks. From left to right: push, hammer, sweep, peg-insert-side, stick push, stick pull, faucet close.}
    \label{fig:metaworld}
\end{figure}

\subsection{Performance on Arcade Learning Environments}
Atari games are suitable for verifying our theory for MDPs with discrete action space.
	\begin{table}[b]
			\centering
			\caption{DQN vs ReMERT on Atari. DQN (Nature) is the performance in the DQN paper~\cite{dqn}. DQN (Baseline) is the performance of our baseline program~\cite{tianshou}.}
			\begin{tabular}{{lccccc}}
			\toprule[1.25pt]
			Method	 & Enduro & KungFuMaster & Kangaroo & MsPacman & Qbert \\
			\hline
			DQN (Nature)  & 301$\pm$24.6 & 23270$\pm$5955 & 6740$\pm$2959 & 2311$\pm$525 & 10596$\pm$3294\\
\hline
DQN (Baseline) & 1185$\pm$100 & 29147$\pm$7280 & 6210$\pm$1007 & 3318$\pm$647 & 13437$\pm$2537 \\
\hline
ReMERT (Ours) & \textbf{1303}$\pm$258 & \textbf{35544}$\pm$8432 & \textbf{7572}$\pm$1794 & \textbf{3481}$\pm$1351 & \textbf{14511}$\pm$1138 \\
\hline
			\end{tabular}
			\label{tab:atari}
	\end{table}
The tested games have a relatively stable temporal ordering of states because the initial state and the terminal state have little randomness, so that the assumption of ReMERT is satisfied. As shown by Tab.~\ref{tab:atari}, ReMERT outperforms DQN in all the selected games. The results also suggest that ReMERT can be applied to environments with high dimensional state spaces. Results of more Atari games are listed in Appendix~\ref{experiments}. We do not include ReMERN for comparison because DisCor which is a composing part of ReMERN has no open-source code available for discrete action space.

\subsection{Demonstration on Key Properties of ReMERN and ReMERT}

\subsubsection{Influence of Environment Randomness}
Fig.~\ref{fig:mujoco} and Fig.~\ref{fig:metaworld} show that ReMERN has a better performance on Meta-World than on Mujoco tasks. We attribute this to the robustness of our strategy in environments with high randomness. For a highly stochastic environment, the estimation of Q value is difficult. When the estimation of Q value is inaccurate, the target Q value is also inaccurate, leading to a suboptimal update process in off-policy RL algorithms. Thanks to the closer value estimation to oracle principle, ReMERN estimates the Q value more accurately than other methods. However, for less stochastic environments like MuJoCo environments, the accuracy of error network might become the bottleneck of ReMERN. 

To show this empirically, we add Gaussian noise to the reward function in MuJoCo environments. The details of the experimental setup are listed in Appendix~\ref{experiments}. Fig.~\ref{fig:noise} show that: (1) ReMERN and ReMERT perform better than SAC in stochastic environments, which verifies our analysis. (2) Though ReMERT suffers statistical error of temporal ordering, it is robust to the randomness of reward because the temporal property is not affected by the noise.

\begin{figure}[t]
    \centering
    \includegraphics[width=0.8\textwidth]{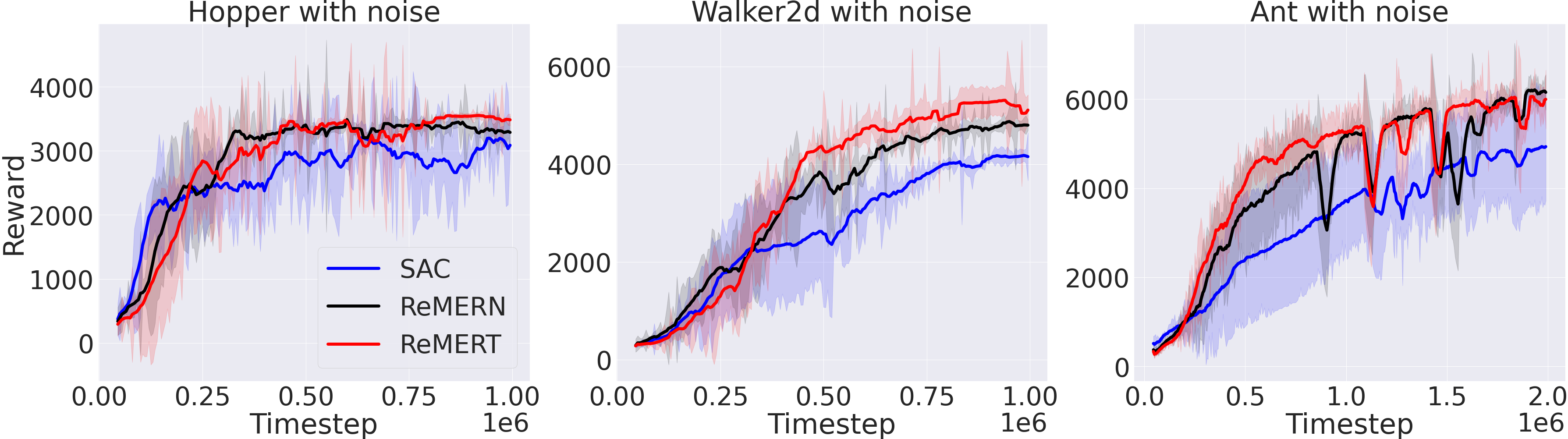}
    \vspace{-0.5mm}
    \caption{Performance of ReMERN, ReMERT and SAC on three continuous control tasks with reward noise.}
    \label{fig:noise}
\end{figure}

\subsubsection{Analysis of TCE on Deterministic Tabular Environments}

\begin{wrapfigure}[15]{r}{0.4\textwidth}
    \centering
    \vspace{-3mm}
    \includegraphics[width=0.4\textwidth]{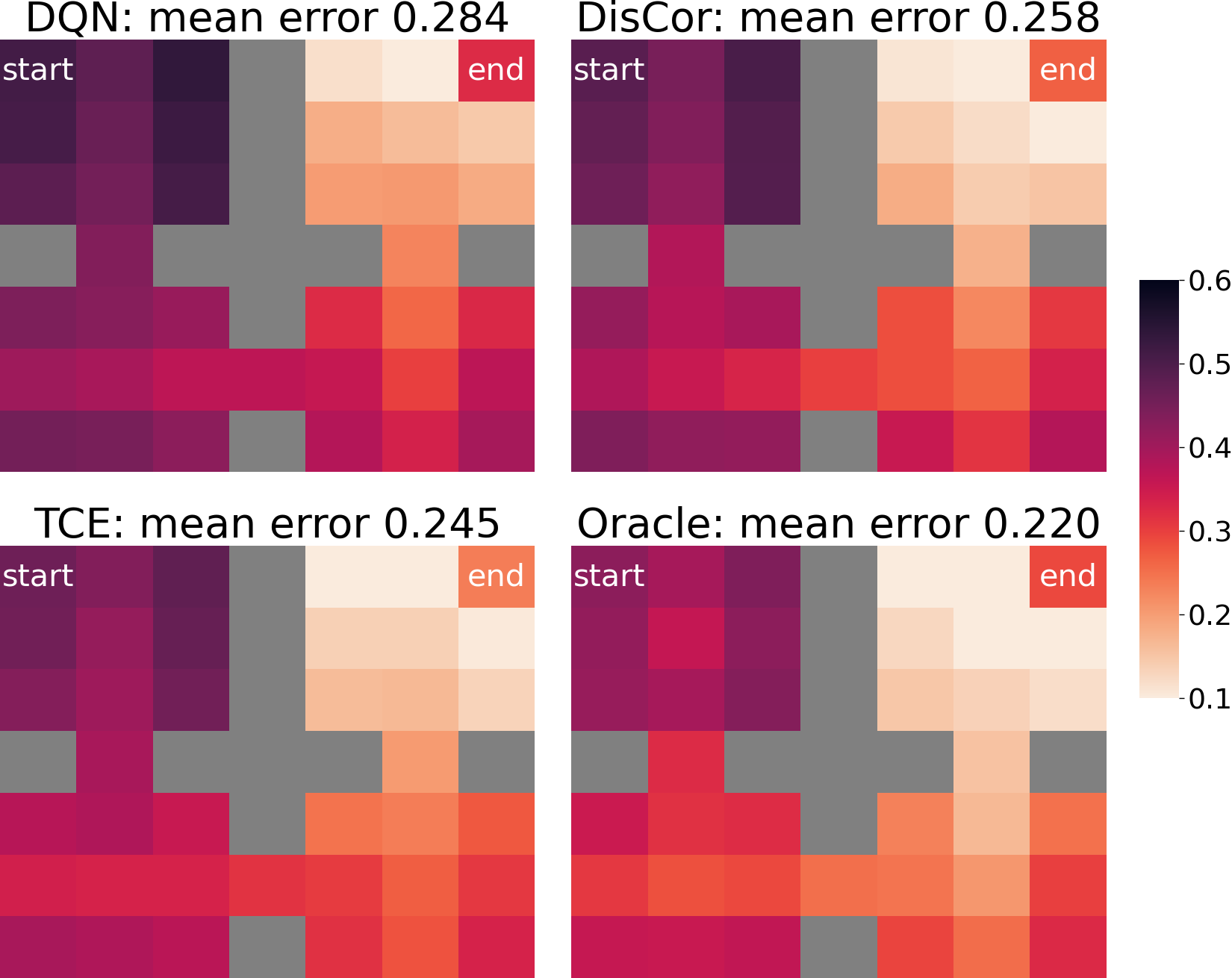}
    \caption{TCE and DisCor in Gridworld}
    \label{fig:TCE}
\end{wrapfigure}

To analyze the effect of the principle behind TCE, we evaluate the Q error in Gridworld with image input. We plot the $|Q_k-Q^*|$ error of standard DQN, DQN with DisCor, DQN with TCE and DQN with oracle at some time in the training process in Fig.~\ref{fig:TCE}.  TCE is combined with DQN to estimate term (c) in Eq.~(\ref{eq_discrete_theorem}) , and the other terms are ignored to compute $w_k$. DQN with oracle uses the ground-truth error $|Q_k-Q^*|$ to calculate the prioritization weight. The result shows that DQN with TCE achieves a more accurate Q value estimator than those of standard DQN and DQN with DisCor, while DQN with oracle $|Q_k-Q^*|$ achieves the most accurate Q value estimator. The lower efficiency of DQN with DisCor is due to the slower convergence speed of the error network. This proves the principle behind our theory effective, and TCE is a decent approximation of $|Q_k-Q^*|$. 

\section{Conclusion and Future Work}
In this work, we first revisit the existing methods of prioritized sampling and point out that the objectives of these methods are different from the objective of RL, which can lead to a suboptimal training process. To solve this issue, we analyze the prioritization strategy from the perspective of 
regret minimization, which is equivalent to return maximization in RL. Our analysis gives a theoretical explanation for some prioritization methods, including PER, LFIW and DisCor. Based on our theoretical analysis, we propose two practical prioritization strategies, ReMERN and ReMERT, that directly aims to improve the policy. ReMERN is robust to the randomness of environments, while ReMERT is more computational efficient and more accurate in environments with a stable temporal ordering of states. 
Our approaches obtain superior results compared to previous prioritized sampling methods. Future work can be conducted in the following two directions. First, the framework to obtain the optimal distribution in off-policy RL can be generalized to model-based RL and offline RL. Second, the two proposed algorithms are suitable for different kinds of MDP, so finding a unified prioritization method for all MDPs can further improve the performance.

\section*{Acknowledgements and Disclosure of Funding}
We thank Xintong Qi and Xiaolong Yin for helpful discussions. We would also like to thank two groups of anonymous  reviewers  for  their  valuable  comments  on our paper. This work is supported by National Key Research and Development Program of China (2020AAA0107200) and NSFC(61876077).

\bibliographystyle{unsrt}
\bibliography{neurips2021}
\newpage
\appendix
\section{Proof of Theorem~\ref{thm_optimization}}\label{proof}
In this section, we present detailed proofs for the theoretical derivation of Thm.~\ref{thm_optimization}, which aims to solve the following optimization problem:
\begin{equation}\label{eq_regret}
\begin{aligned}
    &\min_{w_k} \qquad \qquad \eta(\pi^*)-\eta(\pi_{k})\\
    &~\textnormal{s.t.}\quad Q_k=\argmin_{Q\in\mathcal{Q}}~\mathbb E_{\mu}[w_k(s,a)\cdot(Q-\mathcal B^* Q_{k-1})^2(s,a)],\\
    &\qquad \ \ \mathbb E_\mu[w_k(s,a)]=1, \quad w_k(s,a)\geq 0,
\end{aligned}
\end{equation}

The problem is equivalent to:

\begin{equation}\label{eq_pk}
\begin{aligned}
    &\min_{p_k} \qquad  \eta(\pi^*)-\eta(\pi_k)\\
    &~\textnormal{s.t.} \quad Q_k=\argmin_{Q\in \mathcal{Q}}\mathbb E_{p_k}[(Q-\mathcal B^\pi Q_{k-1})^2(s,a)]\\
    &\qquad \ \ \sum_{s,a}p_k(s,a)=1, \quad p_k(s,a)\geq 0,
\end{aligned}
\end{equation}

The desired $w_k(s,a)$ is $\frac{p_k(s,a)}{\mu(s,a)}$, where $p_k(s,a)$ is the solution to the problem~\ref{eq_pk}.

To solve Problem ~\ref{eq_pk},
we need to give the definition of \textit{total variation distance}, \textit{Wasserstein metric} and the diameter of a set, and introduce some mild assumptions.

\begin{definition}[total variation distance]
 The total variation (TV) distance of distribution $P$ and $Q$ is defined as
 $$D_{\textnormal{TV}}(P, Q)=\frac{1}{2}\left\|P-Q\right\|_1$$
\end{definition}

\begin{definition}[Wasserstein metric]
 For $F, G$ two c.d.fs over the reals, the Wasserstein metric is defined as
 $$
d_{p}(F, G):=\inf _{U, V}\|U-V\|_{p}
$$
where the infimum is taken over all pairs of random variables $(U, V)$ with respective cumulative distributions $F$ and $G$.
\end{definition}

\begin{definition}
 The diameter of a set $A$ is defined as
 $$\textnormal{diam}(A)=\sup_{x,y\in A}m(x,y)$$
 where $m$ is the metric on $A$.
\end{definition}

\begin{assumption}\label{assump_1}
The state space $\mathcal{S}$ and action space $\mathcal{A}$ are metric spaces with a metric $m$.
\end{assumption}

\begin{assumption}\label{assump_2}
The Q function is continuous with respect to $\mathcal{S}\times \mathcal{A}$. 
\end{assumption}

\begin{assumption}\label{assump_3}
The transition function $T$ is continuous with respect to $\mathcal{S}\times \mathcal{A}$ in the sense of Wasserstein metric, i.e., 
$$\lim_{(s,a)\to (s_0,a_0)}d_p(T(\cdot|s,a),T(\cdot|s_0,a_0))=0,$$
where $d_p$ denote the Wasserstein metric.
\end{assumption}

These assumptions are not strong and can be satisfied in most of environments includes MuJoCo, Atari games and so on. 

Let $d_i^\pi(s)$ denote the discounted state distribution, where the state is visited by the agent for the i-th time. that is
$$d^\pi_i(s)=(1-\gamma)\sum_{t_i=0}^{\infty}\gamma^{t_i}\textnormal{Pr}(s_{t_k}=s, \forall k\in [i]),$$
where $[k]=\{j\in \mathbb N_+: j\leq k\}$.
Notably, 
\begin{align}
\label{eq_sum} &d^\pi(s)=\sum_{i=1}^\infty d_i^\pi(s)\\
\label{eq_next} &d^\pi_i(s)=\sum_{t=1}^\infty\rho^\pi(s, \pi(s), t)\gamma^t d^\pi_{i-1}(s),
\end{align}
where $\rho^\pi(s, \pi(s), t)$ is the shorthand for $\mathbb E_{a\sim \pi}\rho^\pi(s, a, t)$.

The standard definitions of Q function, value function and advantage function is:
\begin{align*}
    &Q^\pi(s, a)=\mathbb E_\pi[\sum_{t\geq 0}\gamma^t r(s_t, a_t)|s_0=s, a_0=a].\\
    &V^\pi(s)=\mathbb E_\pi[\sum_{t\geq 0}\gamma^t r(s_t, a_t)|s_0=s].\\
    &A^\pi(s, a)=Q^\pi(s, a)-V^\pi(s).\\
\end{align*}

In the follows, Lemma \ref{lemma_expectation} is a technique used in Lemma \ref{lemma_dpik}. Lemma \ref{lemma_dpik} shows that $\left|\frac{\partial d^\pi(s)}{\partial \pi(s)}\right|$ is a small quantity. 

\begin{lemma}\label{lemma_expectation}
Let $f$ be an Lebesgue integrable function, $P$ and $Q$ are two probability distributions, $|f|\leq C$, then
\begin{equation}
    \left|\mathbb E_{P(x)}f(x)-\mathbb E_{Q(x)}f(x)\right|\leq CD_{\textnormal{TV}}(P, Q)
\end{equation}
\end{lemma}
\begin{proof}
\begin{align*}
    \left|\mathbb E_{P(x)}f(x)-\mathbb E_{Q(x)}f(x)\right|&=\left|\sum_x [P(x)f(x)-Q(x)f(x)]\right|\\
    &=\left|\sum_x [P(x)f(x)-Q(x)f(x)]\mathbb I[P(x)>Q(x)]\right.\\
    &\quad \left.-\sum_x [P(x)f(x)-Q(x)f(x)]\mathbb I[P(x)<Q(x)]\right|\\
    &\leq CD_{\textnormal{TV}}(P, Q)
\end{align*}
\end{proof}

\begin{lemma}\label{lemma_dpik}
Let $\epsilon_\pi=\sup_{s,a}\sum_{t=1}^\infty \gamma^t\rho^\pi(s,a,t)$, we have
\begin{equation}
\left|\frac{\partial d^\pi(s)}{\partial \pi(s)}\right|\leq \epsilon_\pi d_1^\pi(s)
\end{equation}
and $\epsilon_\pi\leq 1$.

\end{lemma}
\begin{proof}
 The definition of $\rho^\pi(s,a,t)$ implies 
 $$0\leq \sum_{t=1}^\infty \gamma^t\rho^\pi(s,a,t)\leq \epsilon_\pi\leq 1, \qquad \forall a\in \mathcal{A}$$
 Based on this fact, we have
 $$\left|\sum_{t=1}^\infty \gamma^t\left(\rho^\pi(s,a_1,t)-\rho^\pi(s,a_2,t)\right)\right|\leq \epsilon_\pi, \qquad \forall a_1, a_2\in \mathcal{A}$$
 
 Let $\rho^\pi(s, \pi(s), t)$ be a shorthand for $\mathbb E_{a\sim \pi(s)}\rho^\pi(s, a, t)$. 

If $\pi$ changes a little and becomes $\pi'$, and $\delta_a=D_{\text{TV}}(\pi(s),\pi'(s))$, then we have
\begin{equation}\label{eq_all}
\begin{aligned}
&\quad \ \left|\sum_{t=1}^\infty\gamma^t\left(\rho^\pi(s, \pi(s), t)-\rho^\pi(s, \pi'(s), t)\right)\right|\\
&= \left|\mathbb E_{a_1\sim \pi}\sum_{t=1}^\infty \gamma^t\rho^\pi(s,a_1,t)-\mathbb E_{a_2\sim \pi'}\sum_{t=1}^\infty \gamma^t\rho^\pi(s,a_1,t)\right|\\
&\leq \epsilon_\pi\delta_a
\end{aligned}
\end{equation}
This inequality comes from Lemma \ref{lemma_expectation}.

We denote the difference between $d_2^\pi(s)$ and $d_2^{\pi'}(s)$ as $\Delta d_2(s)$, which can be bounded as follows:
\begin{align*}
    \Delta d_2(s)&=|d^\pi_2(s)-d_2^{\pi'}(s)|\\
    &= \left|\sum_{t=1}^\infty\gamma^t\left(\rho^\pi(s, \pi(s), t)-\rho^{\pi}(s, \pi'(s), t)\right) d^\pi_{1}(s)\right|\\
    &= d^\pi_{1}(s)\left|\sum_{t=1}^\infty\gamma^t\left(\rho^\pi(s, \pi(s), t)-\rho^{\pi}(s, \pi'(s), t)\right)\right|\\
    &\leq \epsilon_\pi\delta_ad_1^\pi(s)
\end{align*}
Recursively, we have 
$$\Delta d_i(s)\leq\epsilon_\pi^{i-1}\delta_a^{i-1}d_1^\pi(s)$$

Obviously, the change of $\pi$ at state $s$ won't change $d_1^\pi(s)$. According to Eq. (\ref{eq_sum}), 
\begin{align*}
   \Delta d(s)&\leq \sum_{i=1}^\infty \Delta d_i(s)\\ 
   &\leq \sum_{i=2}^\infty(\epsilon_\pi\delta_a)^{i-1} d_1^\pi(s)\\
   &=\frac{\epsilon_\pi\delta_a}{1-\epsilon_\pi\delta_a}d_1^\pi(s)
\end{align*}

According to $\frac{\partial d^\pi(s)}{\partial \pi(s)}=\lim_{\delta_a\to 0}\frac{\Delta d(s)}{\delta_a}$, we have 
$$\left|\frac{\partial d^\pi(s)}{\partial \pi(s)}\right|\leq \epsilon_\pi d_1^\pi(s)$$

This concludes the proof.
\end{proof}

\begin{lemma}\label{lemma_softmax}
Given two policy $\pi_1$ and $\pi_2$, where $\pi_1(a|s)=\frac{\exp(Q_1(s,a))}{\sum_{a'}\exp(Q_1(s,a'))}$. Then 
$$\mathbb E_{a\sim \pi_2}Q_1(s,a)-\mathbb E_{a\sim \pi_1}Q_1(s,a)\leq 1 $$
\end{lemma}

\begin{proof}
Suppose there are two actions $a_1$, $a_2$ under state $s$, and let $Q_1(s,a_1)=u$, $Q_1(s,a_2)=v$. Without loss of generality, let $u\leq v$.
\begin{align*}
    \mathbb E_{a\sim \pi_2}Q_1(s,a)-\mathbb E_{a\sim \pi_1}Q_1(s,a)&\leq v-\frac{ue^u+ve^v}{e^u+e^v}\\
    &=v-\frac{u+ve^{v-u}}{1+e^{v-u}}\\
    &=v-u--\frac{(v-u)e^{v-u}}{1+e^{v-u}}
\end{align*}

Let $f(z)=z-\frac{ze^z}{1+e^z}$, the maximum point $z_0$ of $f(z)$ satisfies $f'(z_0)=0$ where $f'$ is the derivative of $f$, i.e., $\frac{e^{z_0}(1+z_0+e^{z_0}}{(1+e^{z_0})^2}-1=0$. This implies $1+e^{z_0}=z_0e^{z_0}$ and $z_0\in (1,2)$. We have

\begin{align*}
    \mathbb E_{a\sim \pi_2}Q_1(s,a)-\mathbb E_{a\sim \pi_1}Q_1(s,a)&\leq f(v-u)\leq z_0-1 \leq 1
\end{align*}

If the number of action is more than 2 and $Q_1(s,a_1)\geq Q_1(s,a_2)\geq \cdots Q_1(s,a_n)$, let $b_1$ represents $a_1$ and $b_2$ represents all other actions. Then $Q_1(s,b_1)=Q_1(s,a_1)$ and $Q_1(s,b_2)=\sum_{j=2}^n\frac{Q_1(s,a_j)\exp(Q_1(s,a_j)}{\sum_{k=2}^n\exp(Q_1(s,a_k))}$. In this way, we can derive the upper bound of $\mathbb E_{a\sim \pi_2}Q_1(s,a)-\mathbb E_{a\sim \pi_1}Q_1(s,a)$ as above. 
\end{proof}

The following lemma is proposed by Kakade, 
\begin{lemma}[Lemma 6.1 in~\cite{kakade}]\label{lemma_kakade}
For any policy $\tilde \pi$ and $\pi$, 
\begin{equation}
    \eta(\tilde \pi)-\eta(\pi)=\frac{1}{1-\gamma}\mathbb E_{d^{\tilde \pi}(s,a)}[A_\pi(s,a)]
\end{equation}
\end{lemma}

\begin{lemma}\label{lemma_main}
In discrete MDPs, let $\epsilon_{\pi_k}=\sup_{s,a}\sum_{t=1}^\infty \gamma^t\rho^{\pi_k}(s,a,t)$, the optimal solution $p_k$ to a relaxation of optimization problem \ref{eq_pk} satisfies the following relationship:
\begin{equation}\label{eq_sgn}
    p_k(s,a )=\frac{1}{Z^*}\left( D_k(s,a)+\epsilon_k(s,a)\right)
\end{equation}
where $D_k(s,a)=d^{\pi_k}(s,a)(2-\pi_{k}(a|s))\exp \left(-\left|Q_{k}-Q^{*}\right|(s, a)\right) \left|Q_{k}-\mathcal{B}^* Q_{k-1}\right|(s, a)$, $Z^*$ is the normalization constant and $\frac{\epsilon_k(s,a)}{D_k(s,a)}\leq \epsilon_{\pi_k}$.
\end{lemma}

\begin{proof}

Suppose $a^*\sim \pi^*(s)$. Let $\pi=\pi_k$, $\tilde \pi=\pi^*$ in Lemma~\ref{lemma_kakade}, we have
\begin{equation}\label{eq_bound}
\begin{aligned}
    &\quad \ \eta(\pi^*)-\eta(\pi_k)\\
    &=-\frac{1}{1-\gamma}\mathbb E_{d^{\pi_k}(s,a)}A_{\pi^*}(s, a)\\
    &=\frac{1}{1-\gamma}\mathbb E_{d^{\pi_k}(s,a)}(V^*(s)-Q^*(s, a))\\
    &= \frac{1}{1-\gamma}\mathbb E_{d^{\pi_k}(s,a)}\Big(V^*(s)-Q_k(s,a^*)+Q_k(s,a^*)-Q_k(s,a)+Q_k(s,a)-Q^*(s,a)\Big)\\
    &\overset{(a)}{\leq} \frac{1}{1-\gamma}\Big(\mathbb E_{d^{\pi_k}(s)}(Q^*(s, a^*)-Q_k(s,a^*))+\mathbb E_{d^{\pi_k}(s,a)}(Q_k(s, a)-Q^*(s, a))+1\Big)\\
    &\leq \frac{1}{1-\gamma}\Big(\mathbb E_{d^{\pi_k}(s)}\left|Q^*(s, a^*)-Q_k(s,a^*)\right|+\mathbb E_{d^{\pi_k}(s,a)}\left|Q_k(s, a)-Q^*(s, a)\right|+1\Big)\\
    &= \frac{2}{1-\gamma}\Big(\mathbb E_{d^{\pi_k, \pi^*}}\left|Q_k(s,a)-Q^*(s,a)\right|+1\Big),
\end{aligned}
\end{equation}
where $d^{\pi_k, \pi^*}(s,a)=d^{\pi_k}(s)\frac{\pi_k(a|s)+\pi^*(a|s)}{2}$ and (a) uses Lemma \ref{lemma_softmax}.

Since both sides of the above equation have the same minimum (here the minima are given by $Q_k = Q^*$), we can replace the objective in Problem \ref{eq_pk} with the upper bound in Eq.~(\ref{eq_bound}) and solve
the relaxed optimization problem.
\begin{align}
    \label{eq_obj} &\min_{p_k}\qquad \mathbb E_{d^{\pi_k}(s,a)}[|Q_k-Q^*|]\\
    \label{eq_constrain} &~\textnormal{s.t.} \quad Q_k=\argmin_{Q\in \mathcal{Q}}\mathbb E_{p_k}[(Q-\mathcal B^\pi Q_{k-1})^2(s,a)],\\
    &\qquad \ \sum_{s,a}p_k(s,a)=1, \quad p_k(s,a)\geq 0.
\end{align}
Here we use $d^{\pi_k}(s,a)$ to replace $d^{\pi_k, \pi^*}$  because we can not access $\pi^*$, and the best surrogate available is $\pi_k$. 

\textbf{Step 1: Jensen's Inequality.} 
The optimization objective can be further relaxed with Jensen's Inequality, based on the fact that  $f(x)=\exp(-x)$ is a convex function.
\begin{equation}\label{eq_jensen}
\mathbb E_{d^{\pi_k}(s,a)}[|Q_k-Q^*|]=-\log\exp(-\mathbb E_{d^{\pi_k}(s,a)}[|Q_k-Q^*|])\leq -\log\mathbb E_{d^{\pi_k}(s,a)}[\exp(-|Q_k-Q^*|)]
\end{equation}

Similarly, both sides of Eq.~(\ref{eq_jensen}) have the same minimum. We obtain the following new optimization problem by replacing the objective with the upper bound in this equation:
\begin{equation}\label{eq_new_obj}
\begin{split}
    &\min_{p_k} \ -\log\mathbb E_{d^{\pi_k}(s,a)}[\exp(-|Q_k-Q^*|)]\\
    &~\textnormal{s.t.} \quad Q_k=\argmin_{Q\in\mathcal{Q}}\mathbb E_{p_k}[(Q-\mathcal{B}^*Q_{k-1})^2],\\
    &\qquad \ \sum_{s,a}p_k(s,a)=1, \quad p_k(s,a)\geq 0.
\end{split}
\end{equation}

\textbf{Step 2: Computing the Lagrangian.}  In order to optimize problem \ref{eq_new_obj}, we follow the standard procedures of Lagrangian multiplier method. The Lagrangian is:
\begin{equation}
\begin{aligned}
    \mathcal{L}(p_k;\lambda, \mu)=-\log\mathbb E_{d^{\pi_k}(s,a)}[\exp(-|Q_k-Q^*|)]&+\lambda(\sum_{s,a}p_k(s,a)-1)-\mu^Tp_k.\\
\end{aligned}
\end{equation}
where  $\lambda$ and $\mu$ are the Lagrange multipliers. 

\textbf{Step 3: IFT gradient used in the Lagrangian.} $\frac{\partial Q_k}{\partial p_k}$ can be computed according to implicit function theorem (IFT). The IFT gradient is given by:
\begin{equation}\label{eq_gradient}
    \left.\frac{\partial Q_{k}}{\partial p_{k}}\right|_{Q_{k}, p_{k}}=-\left[\operatorname{Diag}\left(p_{k}\right)\right]^{-1}\left[\operatorname{Diag}\left(Q_{k}-\mathcal{B}^{*} Q_{k-1}\right)\right]
\end{equation}
The derivation is similar to that in ~\cite{discor}.

\textbf{Step 4: Approximation of the gradient used in the Lagrangian.} We derive an expression for $\frac{\partial d^{\pi_k}(s,a)}{\partial p_k}$, which will be used when computing the gradient of the Lagrangian. We use $\pi_k$ to denote the policy induced by $Q_k$.

\begin{align*}
    \frac{\partial d^{\pi_k}(s,a)}{\partial p_k}&=\frac{\partial d^{\pi_k}(s,a)}{\partial \pi_k}\frac{\partial \pi_k}{\partial Q_k}\frac{\partial Q_k}{\partial p_k}\\
    &= (d^{\pi_k}(s)+\epsilon_2(s))\frac{\partial \pi_k}{\partial Q_k}\frac{\partial Q_k}{\partial p_k}\\
    &\overset{(b)}{=}(d^{\pi_k}(s)+\epsilon_2(s)\pi_k(a|s)\frac{\sum_{a'\neq a}\exp\left(Q_k(s,a')\right)}{\sum_{a'}\exp(Q_k(s,a'))}\frac{\partial Q_k}{\partial p_k}\\
    &\overset{(c)}{=}d^{\pi_k}(s,a)(1-\pi_k(a|s))\frac{\partial Q_k}{\partial p_k}+\epsilon_2(s)\pi_k(a|s)(1-\pi_k(a|s))\frac{\partial Q_k}{\partial p_k}
\end{align*}

where $\epsilon_2(s)=\frac{\partial d^{\pi_k}(s)}{\partial \pi_k(s)}$. (b) and (c) are based on the fact that $\pi_k(a|s)=\frac{\exp\left(Q_k(s,a)\right)}{\sum_{a'}\exp\left(Q_k(s,a')\right)}$. 

\textbf{Step 5: Computing optimal $p_k$.} By KKT conditions, we have
$$\frac{\partial \mathcal{L}(p_k; \lambda, \mu)}{\partial p_k}=0$$

\begin{align*}
    &\quad \ \frac{\partial \mathcal{L}(p_k; \lambda, \mu)}{\partial p_k}\\
    &=\frac{\exp \left(-\left|Q_{k}-Q^{*}\right|(s, a)\right)}{Z} (d^{\pi_k}(s,a)\operatorname{sgn}\left(Q_{k}-Q^{*}\right) \cdot \frac{\partial Q_{k}}{\partial p_{k}}+\cdot \frac{\partial d^{\pi_k}(s,a)}{\partial p_k})+\lambda-\mu_{s, a}
\end{align*}
where $Z=\mathbb E_{s^{\prime}, a^{\prime}\sim d^{\pi_k}(s,a)} \exp \left(-\left|Q_{k}-Q^{*}\right|\left(s^{\prime}, a^{\prime}\right)\right)$. Substituting the expression of $\frac{\partial Q_k}{\partial p_k}$ and $\frac{\partial d^{\pi_k}(s,a)}{\partial p_k}$ with the results obtained in Step.~3 and Step.~4 respectively, and let $Z_{s,a}=Z(\lambda^*-\mu^*_{s,a})$, we obtain
\begin{equation}
\begin{aligned}
p_{k}(s, a) &= \Big(d^{\pi_k}(s,a)(\textnormal{sgn}(Q_k-Q^*)+1-\pi_k(a|s))\exp \left(-\left|Q_{k}-Q^{*}\right|(s, a)\right) \left|Q_{k}-\mathcal{B}^{*} Q_{k-1}\right|(s, a)\\
&\quad\  +\epsilon_2(s)\pi_k(a|s)(1-\pi_k(a|s))\exp \left(-\left|Q_{k}-Q^{*}\right|(s, a)\right) \left|Q_{k}-\mathcal{B}^{*} Q_{k-1}\right|(s, a)\Big)\frac{1}{Z_{s,a}}
\end{aligned}
\end{equation}

Notably, $Q_k\approx Q^{\pi_k}\leq Q^*$. Thus, $\textnormal{sgn}(Q_k-Q^*)$ always is 1 approximately, so we can simplify this relationship as
\begin{equation}\label{eq_sgn}
\begin{aligned}
    p_k(s,a )&= \frac{1}{Z_{s,a}}\Big(d^{\pi_k}(s,a)(2-\pi_{k}(a|s))\exp \left(-\left|Q_{k}-Q^{*}\right|(s, a)\right) \left|Q_{k}-\mathcal{B}^* Q_{k-1}\right|(s, a)\\
    &\quad +\epsilon_2(s)\pi_k(a|s)(1-\pi_k(a|s))\exp \left(-\left|Q_{k}-Q^{*}\right|(s, a)\right) \left|Q_{k}-\mathcal{B}^{*} Q_{k-1}\right|(s, a)\Big)
\end{aligned}
\end{equation}
The first term is always larger or equal to zero. The second term does not influence the sign of the equation because the absolute value of $\epsilon_2(s)$ is smaller than $d^{\pi_k}(s)$ according to Lemma \ref{lemma_dpik}. Note that Eq. (\ref{eq_sgn}) is always larger or equal to zero. If it is larger than zero then $\mu^*=0$ by the KKT condition. If it is equal to zero, we can let $\mu^*=0$ because the value of $\mu^*$ does not influence $w_k(s,a)$. Without loss of generality, we can let $\mu^*=0$. Then $Z_{s,a}=Z^*=Z\lambda^*$ is a constant with respect to different $s$ and $a$. In this way, we can simplify Eq. (\ref{eq_sgn}) as follows:

$$p_k(s,a)=\frac{1}{Z^*} \left(D_k(s,a)+\epsilon_k(s,a)\right)$$

where $D_k(s,a)=d^{\pi_k}(s,a)(2-\pi_{k}(a|s))\exp \left(-\left|Q_{k}-Q^{*}\right|(s, a)\right) \left|Q_{k}-\mathcal{B}^* Q_{k-1}\right|(s, a)$ and $\epsilon_k(s,a)=\epsilon_2(s)\pi_k(a|s)(1-\pi_k(a|s))\exp \left(-\left|Q_{k}-Q^{*}\right|(s, a)\right) \left|Q_{k}-\mathcal{B}^{*} Q_{k-1}\right|(s, a)$. 

Based on the expression of $D_k(s,a)$ and $\epsilon_k(s,a)$, we have $$\frac{\epsilon_k(s,a)}{D_k(s,a)}=\frac{\epsilon_2(s)(1-\pi_k(a|s))}{d^{\pi_k}(s)(2-\pi_k(a|s))}\leq \epsilon_{\pi_k}$$
The inequality is from \ref{lemma_dpik}. This concludes the proof.
\end{proof}

\begin{theorem}[formal]\label{thm_optimization}
Let $\epsilon_{\pi_k}=\sup_{s,a}\sum_{t=1}^\infty \gamma^t\rho^{\pi_k}(s,a,t)$. Under Assumption \ref{assump_1}, \ref{assump_2} and \ref{assump_3}, if $\frac{d^{\pi_{k}}(s,a)}{\mu(s,a)}$ exists, we have in MDPs with discrete action spaces, the solution $w_k$ to the relaxed optimization problem \ref{eq_regret} is
\begin{equation}\label{eq_discrete_theorem}
    w_k(s,a )=\frac{1}{Z_1^*}\left( E_k(s,a)+\epsilon_{k,1}(s,a)\right). 
\end{equation}
In MDPs with continuous action spaces, the solution is
\begin{equation}\label{eq_continuous_theorem}
    w_k(s,a)=\frac{1}{Z_2^*}\left( F_k(s,a)+\epsilon_{k,2}(s,a)\right).  
\end{equation}
where $$E_k(s,a)=\frac{d^{\pi_k}(s,a)}{\mu(s,a)}(2-\pi_{k}(a|s))\exp \left(-\left|Q_{k}-Q^{*}\right|(s, a)\right) \left|Q_{k}-\mathcal{B}^* Q_{k-1}\right|(s, a)$$
$$F_k(s,a)=2\frac{d^{\pi_k}(s,a)}{\mu(s,a)}\exp \left(-\left|Q_{k}-Q^{*}\right|(s, a)\right) \left|Q_{k}-\mathcal{B}^* Q_{k-1}\right|(s, a),$$
 $Z_1^*$, $Z_2^*$ is the normalization constants and $\max\left\{\frac{\epsilon_{k,1}(s,a)}{E_k(s,a)}, \frac{\epsilon_{k,2}(s,a)}{F_k(s,a)}\right\}\leq \epsilon_{\pi_k}$.

\end{theorem}

\begin{proof}
By Lemma \ref{lemma_main}, for MDPs with discrete action space and state space, we have
$$
p_k(s,a )=\frac{1}{Z^*}\left(D_k(s,a) +\epsilon_k(s,a)\right)
$$
Based on the deviation of Problem \ref{eq_pk}, the solution in this situation is 
\begin{equation}\label{eq_tabular}
 w_k(s,a )=\frac{1}{Z^*}\left( \frac{D_k(s,a)}{\mu(s,a)}+\frac{\epsilon_k(s,a)}{\mu(s,a)}\right)
\end{equation}

The existence of $\frac{d^{\pi_{k}}(s,a)}{\mu(s,a)}$ guarantees the existence of $\frac{D_k(s,a)}{\mu(s,a)}$ and $\frac{\epsilon_k(s,a)}{\mu(s,a)}$. 
Let $E_k(s,a)=\frac{D_k(s,a)}{\mu(s,a)}$ and $\epsilon_{k,1}(s,a)=\frac{\epsilon_k(s,a)}{\mu(s,a)}$, we get Eq. (\ref{eq_discrete_theorem}).

We derive the result for continuous action space and state space as follows, the result for continuous state space and discrete action space, and discrete state space and continuous action space can be derived similarly. 

Remember that $\mathcal{B}^*Q_{k-1}(s,a)=r(s,a)+\gamma\max_{a'}\mathbb E_{s'}Q_{k-1}(s',a')$ and $Q_k(s,a)=\argmin_Q (Q(s,a)-\mathcal{B}^*Q_{k-1}(s,a))^2$, if we use $R(s,a)=Q_k(s,a)-\gamma\max_{a'}\mathbb E_{s'}Q_{k-1}(s',a')$ to replace $r(s,a)$, then $Q_k$ is still the desired Q function after the Bellman update. Since the continuity of $Q_k$, $Q_{k-1}$ and $T$ guarantee $R(s,a)$ is continuous, without loss of generality, we assume $r(s,a)$ is continuous. 

We utilize the techniques in reinforcement learning with aggregated states~\cite{aggregate}. Concretely, we can partition the set of all state-action pairs, with each cell representing
an aggregated state. Such a partition can be defined by a function $\phi:\mathcal S\cup \mathcal A\mapsto \hat{\mathcal{S}}\cup \hat{\mathcal{A}}$, where $\hat{\mathcal{S}}$ is the space of aggregated states and $\hat{\mathcal{A}}$ is the space of aggregated actions. With such a partition, we can discretize the continuous spaces. For example, for the continuous space $\{x\in \mathbb R:0\leq x\leq 10\}$, define $\phi(x)=\sum_{i=1}^9\mathbb I(x\leq x_i)$, and then the space of aggregated states becomes $\{0, 1, 2, \dots, 9\}$, which is a discrete space. 

With function $\phi$, we define the transition function and reward function in this new MDP.
For all $\hat{s}, \hat{s}'\in\hat{\mathcal{S}}$, $\hat{a}\in\hat{\mathcal{A}}$

\begin{equation}
\begin{aligned}
&\hat{T}\left(\hat{s}^{\prime}|\hat{s}, \hat{a}\right) =\frac{\sum_{s,a \in \phi^{-1}(\hat{s}, \hat{a})} \mu(s, a) \sum_{s^{\prime} \in \phi^{-1}\left(\hat{s}^{\prime}\right)} T\left(s'|s, a\right)}{\sum_{s,a \in \phi^{-1}(\hat{s}, \hat{a})} \mu(s, a)} \\
&\hat{r}(\hat{s}, \hat{a}) =\frac{\sum_{s,a \in \phi^{-1}(\hat{s}, \hat{a})} \mu(s, a) r(s, a)}{\sum_{s,a \in \phi^{-1}(\hat{s}, \hat{a})} \mu(s, a)}
\end{aligned}
\end{equation}
where $(\phi(s), \phi(a))$ is simplified as $\phi(s,a)$ and $\phi^{-1}(\hat s, \hat a)$ is the preimage of $(\hat s, \hat a)$.

In this way, Eq. (\ref{eq_tabular}) holds for aggregated state space:
\begin{equation}
    \hat{w}_k(\phi(s,a) )=\frac{1}{\hat{Z}^*}\left( \frac{\hat{D}_k(\phi(s,a))}{\hat{\mu}(\phi(s,a))}+\frac{\hat{\epsilon}_k(\phi(s,a))}{\hat{\mu}(\phi(s,a))}\right)
\end{equation}


Suppose $\hat{\mathcal{S}}$ and $\hat{\mathcal{A}}$ is equipped with metric $m'$, we construct a sequence of functions $\phi_h$, which satisfies

(i) If $m(u_1-u_2)\leq m(u_1-u_3)$, then $m'(\phi_h(u_1)-\phi_h(u_2))\leq m'(\phi_h(u_1)-\phi_h(u_3))$ for all $u_1, u_2, u_3\in \mathcal{S}$ or $u_1, u_2, u_3\in \mathcal{A}$.

(ii) $\lim_{h\to \infty}\textnormal{diam}(\phi_h^{-1}(c))=0$ for all $c\in \mathcal{S}'\cup \mathcal{A}'$.

Based on the two conditions on $\phi_h$ and the continuous of reward function and transition function, for all $ s, s'\in \mathcal{S}$ and $a\in \mathcal{A}$,
\begin{equation}
\begin{aligned}
&\lim_{h\to \infty}\left|\hat{r}(\phi_h(s,a))-r(s,a)\right|=0\\
&\lim_{h\to \infty}\left|\hat{T}(\phi_h(s')|\phi_h(s,a))-T(s'|s,a)\right|=0
\end{aligned}
\end{equation}

This means the constructed MDP approaches the original MDP as $h$ tends to infinity. 

With the Lemma 3 in \cite{abstract},

\begin{align*}
&\lim_{h\to\infty}\mathcal{B}^*\hat{Q}_{k-1}(\phi_h(s,a))=\mathcal{B}^*Q_{k-1}(s,a)\\
&\lim_{h\to\infty}\mathcal{B}^*\hat{Q}^*(\phi_h(s,a))=\mathcal{B}^*Q^*(s,a)
\end{align*}

Note that $Q_k(s,a)=\argmin_Q (Q-\mathcal{B}^*Q_{k-1}(s,a))^2$, $\hat{Q}_k(\phi_h(s,a))=\argmin_Q (Q-\mathcal{B}^*Q_{k-1}(\phi_h(s,a)))^2$, $Q^*(s,a)=\argmin_Q (Q-\mathcal{B}^*Q^*(s,a))^2$ and $\hat{Q}^*(\phi_h(s,a))=\argmin_Q (Q-\mathcal{B}^*Q^*(\phi_h(s,a)))^2$, 
\begin{align*}
    &\lim_{h\to \infty}\hat{Q}_k(\phi_h(s,a))=Q_k(s,a)\\
    &\lim_{h\to \infty}\hat{Q}^*(\phi_h(s,a))=Q^*(s,a)
\end{align*}

Because $\pi(a|s)=\frac{\exp(Q(s,a))}{\sum_{a'}\exp(Q(s,a'))}$, $\pi$ is continuous with respect to $Q$, then we have

\begin{equation*}\label{eq_lim1}
    \lim_{h\to \infty}\hat{\pi}_k(\phi_h(a)|\phi_h(s))=\pi_k(a|s)
\end{equation*}

The continuity of $\pi$ and transition function $T$ guarantees 

$$\lim_{h\to \infty}\hat{d}^{\hat{\pi}_k}(\phi_h(s,a))=d^{\pi_k}(s,a)$$

Therefore, 
\begin{equation}\label{eq_lim2}
\begin{aligned}
    &\lim_{h\to \infty}|\hat{Q}_k-\hat{Q}^*|(\phi((s,a)))=\left|Q_{k}-Q^{*}\right|(s, a)\\
    &\lim_{h\to \infty}|\hat{Q}_k-\hat{\mathcal{B}}^*\hat{Q}^*|(\phi((s,a)))=\left|Q_{k}-\mathcal{B}^* Q_{k-1}\right|(s, a)\\
    &\lim_{h\to \infty} \frac{d^{\hat{\pi}_k}(\phi_h(s,a))}{\hat{\mu}(\phi_h(s,a))}=\frac{d^{\pi_k}(s,a)}{\mu(s,a)}\\
\end{aligned}
\end{equation}

Notably, $\epsilon_2(s)\pi_k(a|s)\leq d^{\pi_k}(s,a)$, the existence of $\frac{d^{\pi_k}(s,a)}{\mu(s,a)}$ implies the existence of $\frac{\epsilon_2(s)\pi_k(a|s)}{\mu(s,a)}$.

\begin{equation}\label{eq_epsilon}
    \lim_{h\to \infty}\frac{\hat{\epsilon}_k(\phi(s,a))}{\hat{\mu}(\phi(s,a))}=\epsilon_{k,1}(s,a)
\end{equation}
where $\epsilon_{k,1}=\frac{\epsilon_k(s)\pi_k(a|s)}{\mu(s,a)}(1-\pi_{k}(a|s))\exp \left(-\left|Q_{k}-Q^{*}\right|(s, a)\right) \left|Q_{k}-\mathcal{B}^* Q_{k-1}\right|(s, a)$.

Using the Eq. (\ref{eq_lim1}), (\ref{eq_lim2}) and (\ref{eq_epsilon}), we have

$$w_k(s,a )=\frac{1}{Z_1^*}\left( E_k(s,a)+\epsilon_{k,1}(s,a)\right). $$

If the action space is continuous, $\pi_k(a|s)=0$, then we have

$$w_k(s,a)=\frac{1}{Z_2^*}\left( F_k(s,a)+\epsilon_{k,2}(s,a)\right)$$

The upper bound of $\frac{\epsilon_{k,1}(s,a)}{E_k(s,a)}$ and $\frac{\epsilon_{k,2}(s,a)}{F_k(s,a)}$ can be derived directly from Lemma \ref{lemma_main}.
This concludes our proof.
\end{proof}

\section{Detailed Proof of Theorem \ref{thm_step}}\label{proof2}

Let $(\mathcal{B}Q)_{k}(s,a)$ denote $|Q_k(s,a)-\mathcal{B}^*Q_k(s,a)|$. We first introduce an assumption.
\begin{assumption}\label{assumption}
At iteration $k$, $(\mathcal{B}Q)_k(s,a)$ is independent of $(\mathcal{B}Q)_k(s',a')$ if $(s,a)\neq (s',a')$  for all $k>0$.
\end{assumption}
This assumption is not strong. If we use a table to represent Q function, it holds apparently. Notably, though we need this assumption in our proof, we can also apply our method on the situation where this assumption doesn't hold. With this assumption, we have the following theorem.

\begin{lemma}\label{thm_cumulative}
Consider a MDP, trajectories $\tau_i=\{s_t^i, a_t^i\}_{t=0}^{T_i}$, $i=0, 1, \dots$ is generated by a policy $\pi$ under this MDP, then we have
\begin{equation}
\begin{aligned}
|Q_k(s,a)-Q^*(s,a)|\leq  &|Q_k(s_t,a_t)-\mathcal{B}^* Q_{k-1}(s_t,a_t)|\\
&+\mathbb E_{\tau}\bigg(\sum_{t'=1}^{{h_\tau^{\pi_k}(s,a)}}\gamma^{t'}\Big((\mathcal{B}Q)_{k-1}(s_{t'},a_{t'})+c\Big)+\gamma^{h_\tau^{\pi_k}(s,a)+1}c\bigg)
\end{aligned}
\end{equation}
where $(\mathcal{B}Q)_k(s_{h_\tau^{\pi_k}(s,a)},a_{h_\tau^{\pi_k}(s,a)})=|Q_k(s_{h_\tau^{\pi_k}(s,a)},a_{h_\tau^{\pi_k}(s,a)})-r(s_{h_\tau^{\pi_k}(s,a)},a_{h_\tau^{\pi_k}(s,a)})|$, $c=\max_{s, a}\big(Q^*(s,a^*)-Q^*(s,a)\big)$, and $(s_{t'}, a_{t'})$ is the $t'$-th state-action pair behind $(s,a)$.
\end{lemma}

\begin{proof}

\begin{align*}
    &\quad \ |Q_k(s_t,a_t)-Q^*(s_t, a_t)|\\
    &=|Q_k(s_t, a_t)-\mathcal{B}^* Q_{k-1}(s_t,a_t)+\mathcal{B}^* Q_{k-1}(s_t,a_t)-\mathcal{B}^* Q^*(s_t,a_t)]|\\
    &\overset{(a)}{\leq}|Q_k(s_t,a_t)-\mathcal{B}^* Q_{k-1}(s_t,a_t)|\\
    &\quad +\gamma |\mathbb E_{p(\tau)}[Q_{k-1}(s_{t+1}, a_{t+1})-Q^*(s_{t+1}, a_{t+1})+Q^*(s_{t+1}, a_{t+1})-Q^*(s_{t+1}, a^*)]|\\
    &\overset{(b)}{\leq} |Q_k(s_t,a_t)-\mathcal{B}^* Q_{k-1}(s_t,a_t)|+\gamma c+\gamma \mathbb E_{\tau}[|Q_{k-1}(s_{t+1}, a_{t+1})-Q^*(s_{t+1}, a_{t+1})|] 
\end{align*}
where the expectation is taken over $s'\sim P(s'|s,a)$, $a'\sim \pi(a'|s')$. (a) uses triangle inequality, (b) is because $f(x)=|x|$ is convex function and using Jensen's Inequality.

Similarly, we have
\begin{align*}
    &\qquad |Q_{k-1}(s_{t+1}, a_{t+1})-Q^*(s_{t+1}, a_{t+1})|\\
    &=|Q_{k-1}(s_{t+1}, a_{t+1})-\mathcal{B}^* Q_{k-1}(s_{t+1},a_{t+1})+\mathcal{B}^* Q_{k-1}(s_{t+1},a_{t+1})-\mathcal{B}^* Q^*(s_{t+1},a_{t+1})]|\\
    &\leq (\mathcal{B}Q)_{k-1}(s_{t+1},a_{t+1})+\gamma c +\gamma \mathbb E_{\tau}[|Q_{k-1}(s_{t+2}, a_{t+2})-Q^*(s_{t+2}, a_{t+2})|] 
\end{align*}
Recursively, 
\begin{equation}
\begin{aligned}
&\quad \ |Q_k(s,a)-Q^*(s,a)|\\
&\leq |Q_k(s_t,a_t)-\mathcal{B}^* Q_{k-1}(s_t,a_t)|+\sum_{t'=1}^{{h_\tau^{\pi_k}(s,a)}}\gamma^{t'}\Big((\mathcal{B}Q)_{k-1}(s_{t'},a_{t'})+c\Big)+\gamma^{h_\tau^{\pi_k}(s,a)+1}c\\
\end{aligned}
\end{equation}
where  $(\mathcal{B}Q)_{k-1}(s_{h_\tau^{\pi_k}(s,a)},a_{h_\tau^{\pi_k}(s,a)})=|Q_{k-1}(s_{h_\tau^{\pi_k}(s,a)},a_{h_\tau^{\pi_k}(s,a)})-r(s_{h_\tau^{\pi_k}(s,a)},a_{h_\tau^{\pi_k}(s,a)})|$.
\end{proof}

This theorem shows that the cumulative Bellman error with a constant $c$ is an upper bound of $|Q_k-Q^*|$, so we can use Bellman error with the constant to estimate this quantity. 

Suppose the Q function is equipped with a learning rate $\alpha$, i.e., $Q_k=\alpha(\mathcal{B}^*Q_{k-1}-Q_{k-1})+(1-\alpha)Q_{k-1} $, we have the following lemma,

\begin{lemma}\label{lemma_B}
\begin{equation}
\begin{aligned}
    &\left\|\mathcal{B}^*Q_k-Q_k\right\|_\infty\leq (\alpha\gamma+1-\alpha)^k\left\|\mathcal{B}^*Q_0-Q_0\right\|_\infty\\
    &\left\|\mathcal{B}^\ast Q_{k-1}-Q_k\right\|_\infty\leq (1-\alpha)(\alpha\gamma+1-\alpha)^{k-1}\left\|\mathcal{B}^\ast Q_0-Q_0\right\|_\infty
\end{aligned}
\end{equation}
\end{lemma}
\begin{proof}
\begin{align*}
    &\quad \quad Q_k=Q_{k-1}+\alpha (\mathcal{B}^*Q_{k-1}-Q_{k-1})\\
    &\Longrightarrow \mathcal{B}^*Q_{k-1}-Q_k=\frac{1-\alpha}{\alpha}(Q_k-Q_{k-1})\\
\end{align*}
\begin{equation}\label{eq_B}
\begin{aligned}
    \left\|\mathcal{B}^*Q_{k}-Q_k\right\|_\infty&\leq \left\|\mathcal{B}^*Q_{k}-\mathcal{B}^*Q_{k-1}\right\|_\infty+\left\|\mathcal{B}^*Q_{k-1}-Q_k\right\|_\infty\\
    &\leq \gamma \left\|Q_k-Q_{k-1}\right\|_\infty+\left\|\mathcal{B}^*Q_{k-1}-Q_k\right\|_\infty\\
    &\leq (\gamma+\frac{1-\alpha}{\alpha})\left\|Q_k-Q_{k-1}\right\|_\infty\\
    &\leq (\alpha\gamma+1-\alpha) \left\|\mathcal{B}^*Q_{k-1}-Q_{k-1}\right\|_\infty
\end{aligned}
\end{equation}

\begin{equation}\label{eq_A}
\begin{aligned}
\left\|Q_k-\mathcal{B}^\ast Q_{k-1}\right\|_\infty&\leq (1-\alpha)\left\|Q_{k-1}-\mathcal{B}^\ast Q_{k-1}\right\|_\infty\\
&\leq (1-\alpha)\left(\left\|Q_{k-1}-\mathcal{B}^\ast Q_{k-2}\right\|_\infty+\left\|\mathcal{B}^\ast Q_{k-2}-\mathcal{B}^\ast Q_{k-1}\right\|_\infty\right)\\
&\leq (1-\alpha)\left(\gamma\left\|Q_{k-2}-Q_{k-1}\right\|_\infty+\left\|Q_{k-1}-\mathcal{B}^\ast Q_{k-2}\right\|_\infty\right)\\
&\overset{(a)}{\leq} (1-\alpha)(\gamma+\frac{1-\alpha}{\alpha})\left\|Q_{k-1}-Q_{k-2}\right\|_\infty\\
&\leq (1-\alpha)(\alpha\gamma+1-\alpha)\left\|\mathcal{B}^\ast Q_{k-2}-Q_{k-2}\right\|_\infty
\end{aligned}
\end{equation} 

Then we can finish the proof by recursively applying Eq. (\ref{eq_B}) and (\ref{eq_A}).
\end{proof}

\begin{lemma}[Azuma]\label{Azuma}
Let $X_0,X_1, \dots$ be a martingale such that, for all $k\geq 1$,
$|X_k-X_{k-1}|\leq c_k$,
Then
\begin{equation}
    \textnormal{Pr}[|X_n-X_0|\geq t]\leq 2\exp(-\frac{t^2}{2\sum^n_{k=1}c^2_k}).
\end{equation}
\end{lemma}

In the follows, we denote $\sum_{t=1}^{h^{\pi_k}_\tau(s,a)}\gamma^{t}(\mathcal{B}Q)_k(s_{t},a_{t})$ as $\mathcal{B}(s, a, k)$.

\begin{lemma}\label{thm_expectation}
Let $\phi_k=(\alpha\gamma+1-\alpha)^k||\mathcal{B}^*Q_0-Q_0||_\infty$, $f(t)=\frac{\gamma-\gamma^{t+1}}{1-\gamma}$ and $\epsilon_{\pi_k}=\sup_{s,a}\sum_{t=1}^{\infty} \gamma^t\rho^{\pi_k}(s,a,t)$.
Under Assumption \ref{assumption},  with probability at least $1-\delta$,

\begin{equation}\label{eq_confidence}
|\mathcal{B}(s,a,k)-f(h^{\pi_k}_\tau(s,a))\mathbb E[(\mathcal{B}Q)_k(s_{t}, a_{t})]|
\leq \sqrt{2f(h^{\pi_k}_\tau(s,a))^2(1+\epsilon_{\pi_k})^2\phi_k^2 \log \frac{2}{\delta}}.
\end{equation}
\end{lemma}

\begin{proof}

Let 
$\mathcal{F}_h=\sigma_t(s_0, a_0, r_0, \dots, s_{h-1}, a_{h-1}, r_{h-1})$
be the $\sigma$-field summarising the information available just before $s_t$ is observed.

Define $Y_h=\mathbb E[\mathcal{B}(s, a, k)|\mathcal{F}_h]$, then $Y_h$ is a martingale because
\begin{align*}
\mathbb E[Y_h|\mathcal{F}_{h-1}]=\mathbb E[\mathbb E[\mathcal{B}(s, a, k)|\mathcal{F}_h]|\mathcal{F}_{h-1}]=\mathbb E[\mathcal{B}(s, a, k)|\mathcal{F}_{h-1}]=Y_{h-1}
\end{align*}

$$
\begin{aligned}
|Y_h-Y_{h-1}|&\leq \gamma^h(1+\epsilon_{\pi_k})\left\|\mathcal{B}^*Q_{k}-Q_k\right\|_\infty\\
&\leq \gamma^h(1+\epsilon_{\pi_k})(\alpha\gamma+1-\alpha)^k\left\|\mathcal{B}^*Q_0-Q_0\right\|_\infty=\gamma^h(1+\epsilon_{\pi_k})\phi_k
\end{aligned}
$$
By Azuma's lemma, 

\begin{equation*}
    \textnormal{Pr}\bigg(|\mathcal{B}(s, a, k)-\mathbb E[\mathcal{B}(s, a, k)]|\geq \sqrt{2\Big(\frac{\gamma-\gamma^{h_\tau^{\pi_k}+1}}{1-\gamma}\Big)^2(1+\epsilon_{\pi_k})^2\phi_k^2 \log \frac{2}{\delta}}\bigg)\leq \delta
\end{equation*}


\end{proof}
Since $(\alpha\gamma+1-\alpha)$ is less than 1, $\phi_k$ decreases exponentially as $k$ increases. This theorem shows that we can use the average Bellman error as a surrogate of Bellman error at specific state-action pair without losing too much accuracy. In this way, $|Q_k-Q^*|(s,a)$ is merely related to the distance to end of the state-action pair. 

\begin{theorem}[formal]
Under Assumption \ref{assumption}, with probability at least $1-\delta$, we have
\begin{equation}
\begin{aligned}
&\quad \ |Q_k(s,a)-Q^*(s,a)|\\
&\leq \mathbb E_{\tau}\bigg(f(h^{\pi_k}_\tau(s,a))\big(\mathbb E[(\mathcal{B}Q)_k(s_{t'}, a_{t'})]+c\big)+\gamma^{h^{\pi_k}_\tau(s,a)+1}c\bigg)+g(k, \delta)
\end{aligned}
\end{equation}
where $g(k, \delta)=(1-\alpha)\phi_{k-1}+\sqrt{2f(h^{\pi_k}_\tau(s,a))^2(1+\epsilon_{\pi_k})^2\phi_k^2 \log \frac{2}{\delta}}$ .
\end{theorem}
\begin{proof}

According to Lemma \ref{thm_cumulative}, we have
\begin{equation}\label{eq_final}
\begin{aligned}
|Q_k(s,a)-Q^*(s,a)|\leq  &|Q_k(s_t,a_t)-\mathcal{B}^* Q_{k-1}(s_t,a_t)|\\
&+\mathbb E_{\tau}\bigg(\sum_{t'=1}^{{h_\tau^{\pi_k}(s,a)}}\gamma^{t'}\Big((\mathcal{B}Q)_{k-1}(s_{t'},a_{t'})+c\Big)+\gamma^{h_\tau^{\pi_k}(s,a)+1}c\bigg)
\end{aligned}
\end{equation}
Using Lemma \ref{lemma_B}, we can upper bound $|Q_k(s_t,a_t)-\mathcal{B}^* Q_{k-1}(s_t,a_t)|$ as $(1-\alpha)\phi_{k-1}$. 
With Lemma \ref{thm_expectation}, $\sum_{t=1}^{h^{pi_k}_\tau(s,a)}\gamma^{t}(\mathcal{B}Q)_k(s_{t},a_{t})$ can be bounded by right hand side of Eq. (\ref{eq_confidence}) with probability $1-\delta$. Substitute the bounds into Eq. (\ref{eq_final}), we have
\begin{equation*}
\begin{aligned}
|Q_k(s,a)-Q^*(s,a)|&\leq  (1-\alpha)\phi_{k-1}+\sqrt{2f(h^{\pi_k}_\tau(s,a))^2(1+\epsilon_{\pi_k})^2\phi_k^2 \log \frac{2}{\delta}}\\
&\quad \ +\mathbb E_{\tau}\bigg(f(h^{\pi_k}_\tau(s,a))\big(\mathbb E[(\mathcal{B}Q)_k(s_{t'}, a_{t'})]+c\big)+\gamma^{h^{\pi_k}_\tau(s,a)+1}c\bigg)\\
&\leq g(k,\delta)\\
&\quad \ +\mathbb E_{\tau}\bigg(f(h^{\pi_k}_\tau(s,a))\big(\mathbb E[(\mathcal{B}Q)_k(s_{t'}, a_{t'})]+c\big)+\gamma^{h^{\pi_k}_\tau(s,a)+1}c\bigg)
\end{aligned}
\end{equation*}
\end{proof}

\section{Algorithms}\label{algorithms}
\begin{algorithm}[htb!]
\caption{ReMERN}
\label{alg:discor}
\begin{algorithmic}[1]
\STATE Initialize Q-values $Q_\theta(s, a)$, a replay buffer $\mu$, an {\color{red} error model $\Delta_\phi(s, a)$}, and a {\color{red}weight model $\kappa_\psi$}.
\FOR{step $k$ in $\{1,\dots , N\}$}
\STATE Collect $M$ samples using $\pi_k$, add them to replay buffer $\mu$, sample $\{(s_i, a_i)\}_{i=1}^N\sim \mu$.
\STATE Evaluate $Q_\theta(s,a)$, $\Delta_\phi(s,a)$ and $\kappa_\psi(s,a)$ on samples $(s_i,a_i)$.
\STATE Compute target values for $Q$ and $\Delta$ on samples:\\
       $y_i=r_i+\gamma\max_{a'}Q_{k-1}(s'_i,a')$.\\
       $\hat{a}_i=\argmax_aQ_{k-1}(s'_i,a)$.\\
       $\hat{\Delta}=|Q_\theta(s,a)-y_i|+\gamma\Delta_{k-1}(s'_i,\hat{a}_i)$.
\STATE {\color{red}Optimize $\kappa_\psi$ using \\
$$
L_{\kappa}(\psi):=\mathbb{E}_{\mathcal{D}_{\mathrm{s}}}\left[f^{*}\left(f^{\prime}\left(\kappa_{\psi}(s, a)\right)\right)\right]-\mathbb{E}_{\mathcal{D}_{\mathrm{f}}}\left[f^{\prime}\left(\kappa_{\psi}(s, a)\right)\right].
$$}
\STATE {\color{red}Compute $w_k$ using\\
$$
w_k(s, a)\propto \frac{d^{\pi_k}(s, a)}{\mu(s, a)}\exp \left(-\gamma\left[P^{\pi^{w_{k-1}}} \Delta_{k-1}\right](s, a)\right).
$$}
\STATE Minimize Bellman error for $Q_\theta$ weighted by $w_k$.\\
       $\theta_{k+1} \leftarrow \underset{\theta}{\operatorname{argmin}} \frac{1}{N} \sum_{i}^{N} {\color{red}w_{k}\left(s_{i}, a_{i}\right)}\left(Q_{\theta}\left(s_{i}, a_{i}\right)-y_{i}\right)^{2}$.
\STATE {\color{red}Minimize ADP error for training $\phi$.\\
       $\phi_{k+1} \leftarrow \underset{\phi}{\operatorname{argmin}} \frac{1}{N} \sum_{i=1}^{N}\left(\Delta_{\phi}\left(s_{i}, a_{i}\right)-\hat{\Delta}_{i}\right)^{2}$.}
\ENDFOR
\end{algorithmic}
\end{algorithm}
\newpage

\begin{algorithm}[htb!]
\caption{ReMERT}
\label{alg:TCE}
\begin{algorithmic}[1]
\STATE Initialize Q-values $Q_\theta(s, a)$, a replay buffer $\mu$, and a {\color{red}weight model $\kappa_\psi$}.
\FOR{step $k$ in $\{1,\dots , N\}$}
\STATE Collect $M$ samples using $\pi_k$, add them to replay buffer $\mu$, sample $\{(s_i, a_i)\}_{i=1}^N\sim \mu$.
\STATE Evaluate $Q_\theta(s,a)$ and $\kappa_\psi(s,a)$ on samples $(s_i,a_i)$.
\STATE Compute target values for $Q$ on samples:\\
       $y_i=r_i+\gamma\max_{a'}Q_{k-1}(s'_i,a')$.\\
       $\hat{a}_i=\argmax_aQ_{k-1}(s'_i,a)$.
\STATE {\color{red}Optimize $\kappa_\psi$ using\\
$$
L_{\kappa}(\psi):=\mathbb{E}_{\mathcal{D}_{\mathrm{s}}}\left[f^{*}\left(f^{\prime}\left(\kappa_{\psi}(s, a)\right)\right)\right]-\mathbb{E}_{\mathcal{D}_{\mathrm{f}}}\left[f^{\prime}\left(\kappa_{\psi}(s, a)\right)\right].
$$}
\STATE {\color{red}Compute $w_k$ using\\
$$
    w_k(s,a)\propto \frac{d^{\pi_k}(s, a)}{\mu(s, a)}\exp \Big(-\mathbb E_{q_{k-1}(\tau)}\text {TCE}_c(s,a)\Big).
    $$}
\STATE Minimize Bellman error for $Q_\theta$ weighted by $w_k$.\\
       $\theta_{k+1} \leftarrow \underset{\theta}{\operatorname{argmin}} \frac{1}{N} \sum_{i}^{N} {\color{red}w_{k}\left(s_{i}, a_{i}\right)}\left(Q_{\theta}\left(s_{i}, a_{i}\right)-y_{i}\right)^{2}$.
\ENDFOR
\end{algorithmic}
\end{algorithm}
\section{Experiments}\label{experiments}
We now present some additional experimental results and experiment details which we could not present due to shortage of space in the main body.
\subsection{Cumulative Recurring Probability on Atari Games}
\begin{table}[h]\label{tab_recurrence}
\centering
\caption{The value of $\epsilon_\pi$ with different policies in Atari games. }
\begin{tabular}{lccc}
\hline
\multicolumn{1}{c}{} & Initial (Random) policy & Policy at timestep 100k & Policy at timestep 200k \\ \hline
Pong                 & 0.00             & 0.00                       & 0.00                       \\ \hline
Breakout             & 0.00             & 0.00                       & 0.00                       \\ \hline
Kangaroo             & 0.44            & 0.32                      & 0.15                      \\ \hline
KungFuMaster         & 0.66            & 0.06                       & 0.01                       \\ \hline
MsPacman             & 0.44            & 0.04                       & 0.00                       \\ \hline
Qbert                & 0.02             & 0.05                       & 0.00                       \\ \hline
Enduro               & 0.00             & 0.00                       & 0.00                       \\ \hline
\end{tabular}
\end{table}
In Pong, Breakout and Enduro, $\epsilon_\pi$ keeps zero, so there is no error terms in such environments. For KungFuMaster and MsPacman, though $\epsilon_\pi$ is high for the initial policy, its value decreases rapidly as the policy updates. The error term in Kangaroo induces some error but $\epsilon_\pi$ is still much smaller than one. 
The experiment results imply we can ignore the error term in most reinforcement learning environments.

\subsection{Illustrations on Stable Temporal Structure}
We conduct an extra experiment in the GridWorld environment to support our claim that the trajectories have a stable temporal ordering of states. Fig.~\ref{fig:ass2} shows an empirical verification of the stable temporal ordering of states property. 

The result shows that the variance of distance to end in one state is not large and decreases fast in training process. This means the property is not a strong assumption and can be satisfied in many environments. 

\begin{figure}[h]
    \centering
    \includegraphics[width=0.5\textwidth]{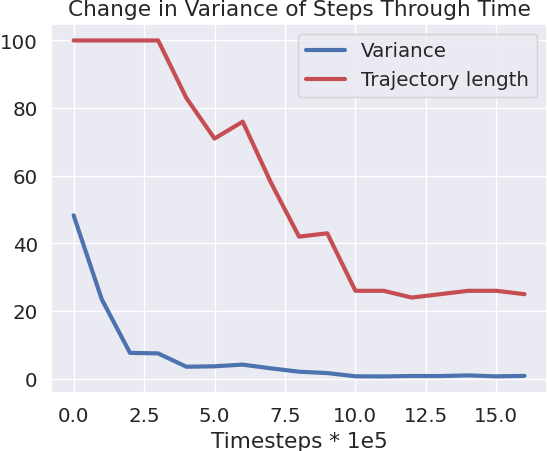}
    \caption{Change in variance of distance to end through time. 
    For each timestep, the red line shows the average trajectory length in the last 500 states. The blue line shows the average variance of the last 500 states, where the variance for each state is calculated from its positions in their corresponding trajectories.}
    \label{fig:ass2}
\end{figure}


\subsection{Description of Involved Environments}
The Meta-World benchmark~\cite{metaworld} includes a series of robotic manipulation tasks. These tasks differ from traditional goal-based ones in that the target objects of the robot. For example, the screw in the hammer task has randomized positions and can not be observed by RL agents. Therefore, Meta-World suite can be highly challenging for current state-of-the-art off policy RL algorithms. Visual descriptions for the Meta-World tasks are shown in Fig.~\ref{fig:metaworld4}. DisCor \cite{discor} showed preferable performance on some Meta-World tasks compared to SAC and PER \cite{per}, but the learning process is slow and unstable.

\begin{figure}[h]
    \centering
    \includegraphics[width=\textwidth]{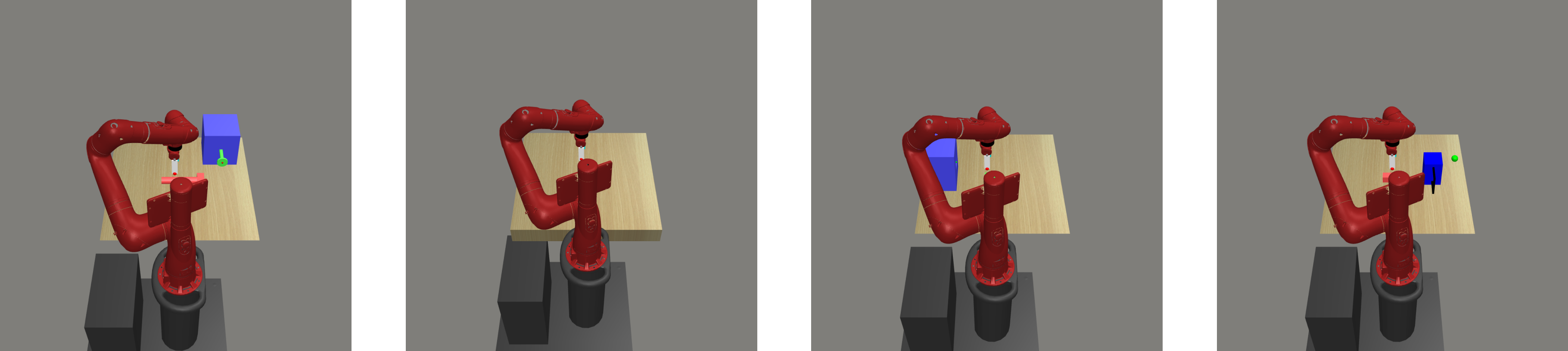}
    \caption{Pictures for Meta-World tasks hammer, sweep, peg-insert-side and stick-push.}
    \label{fig:metaworld4}
\end{figure}

\subsection{Extended Results on Atari Environment}
We evaluate ReMERN on an extended collection of Atari environments. As is shown in Tab.~\ref{tab:atari_extend}, ReMERN outperforms baseline methods in most of the environments.

	\begin{table}[b]
			\centering
			\caption{Extended experiments on Atari.}
			\begin{tabular}{{lcccc}}
			\toprule[1.25pt]
			Environments	 & DQN(Nature) & DQN(Baseline) & PER(rank-b.) & ReMERT(Ours)  \\
			\hline
            Assault & 3395$\pm$775 & 8260$\pm$2274 & 3081 & \textbf{9952}$\pm$3249\\
            \hline
            BankHeist & 429$\pm$650 & 1116$\pm$34 & 824 & \textbf{1166}$\pm$82\\
            \hline
            BeamRider & 6846$\pm$1619 & 5410$\pm$1178 & \textbf{12042} & 5542$\pm$1577\\ 
            \hline
            Breakout & 401$\pm$27 & 242$\pm$79 & \textbf{481} & 223 $\pm$79\\
            \hline
            Enduro & 302$\pm$25 & 1185$\pm$100 & 1266 & \textbf{1303}$\pm$258\\
            \hline
            Kangaroo & 6740$\pm$2959 & 6210$\pm$1007 & \textbf{9053} & 7572$\pm$1794\\
            \hline
            KungFuMaster & 23270$\pm$5955 & 29147$\pm$7280 & 20181 & \textbf{35544}$\pm$8432 \\
            \hline
            MsPacman & 2311$\pm$525 & 3318$\pm$647 & 964.7 & \textbf{3481}$\pm$1350\\
            \hline
            Riverraid & 8316$\pm$1049 & 9609$\pm$1293 & 10205 & \textbf{10215}$\pm$1815 \\
            \hline
            SpaceInvaders & \textbf{1976}$\pm$893 & 925$\pm$371 & 1697 & 877$\pm$249\\
            \hline
            UpNDown & 8456$\pm$3162 & 134502$\pm$68727 & 16627 & \textbf{145235}$\pm$94643 \\
            \hline
            Qbert & 10596$\pm$3294 & 13437$\pm$2537 & 12741 & \textbf{14511}$\pm$1138\\
            \hline
            Zaxxon & 4977$\pm$1235 & 5070$\pm$997 & \textbf{5901} & 5738$\pm$1296 \\
            \hline
			\end{tabular}
			\label{tab:atari_extend}
	\end{table}
	
\subsection{Extended Evaluation on Gridworld}\label{sec_gridworld}
Aside from the FourRooms environment in Gridworld, we also conduct comparative evaluation on the Maze environment. The results are shown in Fig.~\ref{fig:allgrid}. The Maze environment perfectly fits for our TCE-based prioritization, and TCE achieves the best performance among other methods.

\begin{figure}[h]
    \centering
    \includegraphics[width=\textwidth]{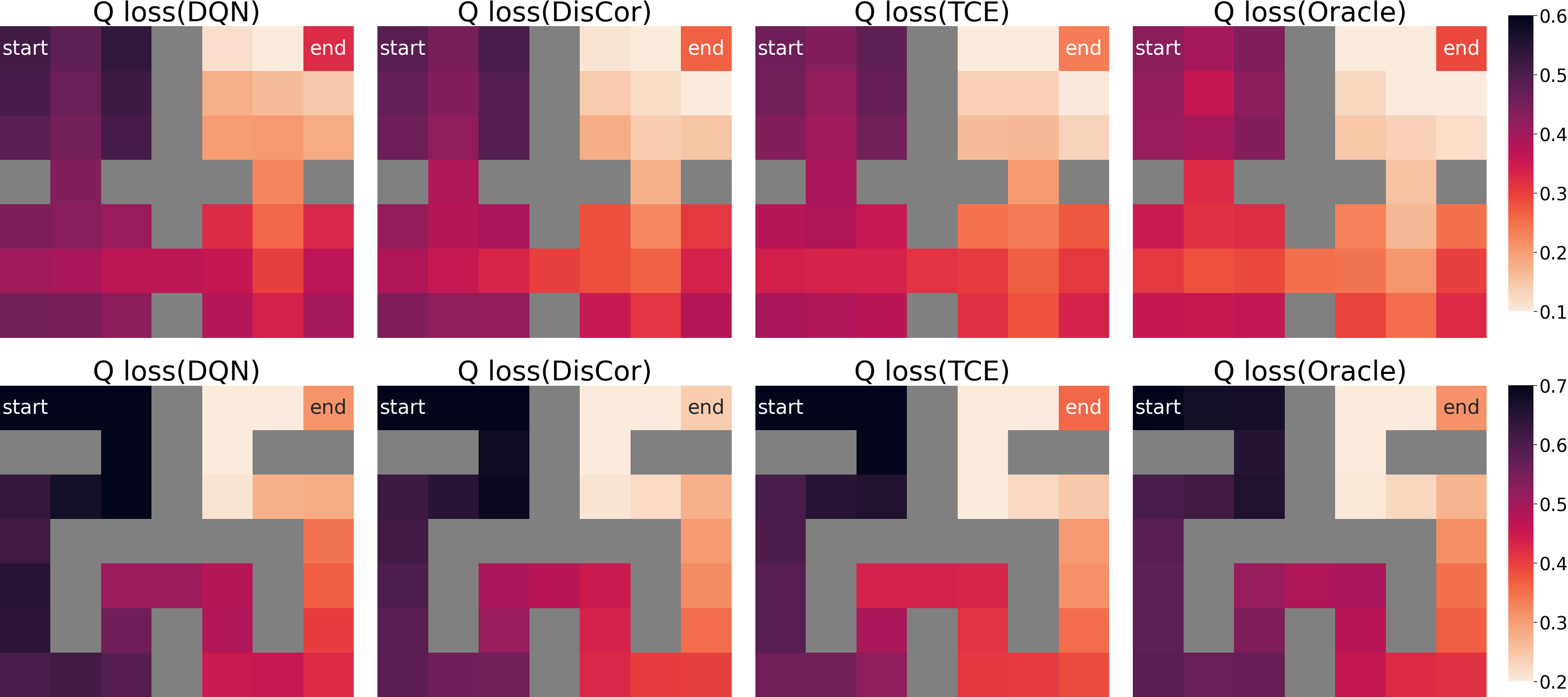}
    \caption{Extended evaluation results on Gridworld.}
    \label{fig:allgrid}
\end{figure}




\subsection{The Relation Between Distance to End and $|Q_k-Q^*|$}
\label{sec:eq_q}

\begin{figure}[h]
    \includegraphics[width=\textwidth]{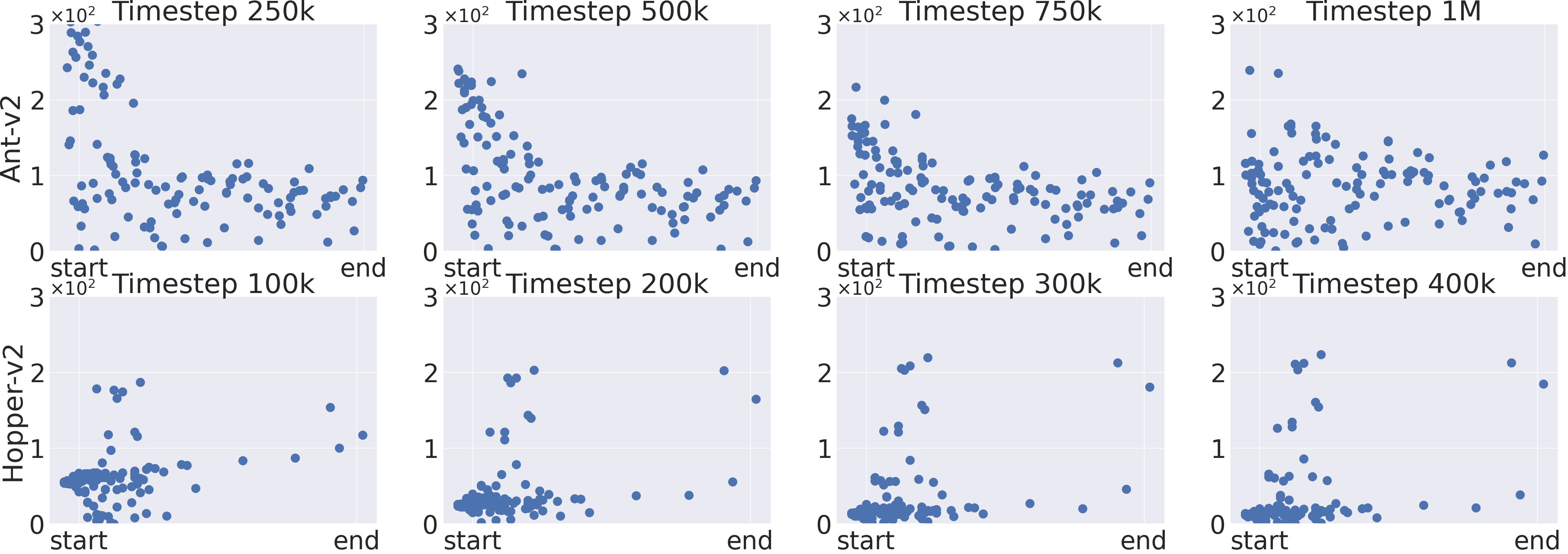}
    \caption{The relationship between $|Q_k-Q^*|$ and distance to end in two MuJoCo tasks (Ant and Hopper).}
    \label{fig:relation}
\end{figure}

In section \ref{sec_gridworld}, the relationship between $|Q_k-Q^*|$ and distance to end has been shown in tabular environments. In this section, we explore the relationship in environments with continuous state and action spaces, i.e., Ant and Hopper tasks of MuJoCo environment. Since $Q^*$ is inaccessible in these complex continuous control tasks, we approximate it by doing Monte-Carlo rollout using the best policy during training. The results are shown in Fig.~\ref{fig:relation}. 

The negative correlation between the two quantities is obvious in Ant-v2, but vague in Hopper-v2. It is because Hopper is a relatively easy task so that all state-action pair have small Q loss and don't have such correlation. The performance of ReMERT shown in Section 4 accords with this observation. ReMERT outperforms other algorithms in environments with a high correlation between the two quantities, and has a relatively poor performance in environments without such correlation. 

\subsection{Implementation Details}

\subsubsection{Algorithm Details}
\paragraph{Weight Normalization} To stabilize the prioritization, we apply normalization to the estimation of two terms: $\frac{d^{\pi_k}(s,a)}{\mu(s,a)}$ and $\exp(-|Q_k-Q^*|)$. 

First, we introduce the normalization in calculating $\frac{d^{\pi_k}(s,a)}{\mu(s,a)}$, which aims to address the finite sample size issue. The normalization is:
$$
\tilde{\kappa}_\psi(s,a):=\frac{\kappa_\psi(s,a)^{1/T}}{\mathbb E_{\mathcal{D}_s}[\kappa_\psi(s,a)^{1/T}]}
$$
where $\mathcal{D}_s$ is the slow buffer and $T$ is temperature.

ReMERN uses $\Delta_\phi$ to fit the discounted cumulative Bellman error. However, the Bellman error has different scales in various environments, leading to erroneous weight. We normalize it by dividing a moving average of Bellman error. The divisor is denoted as $\tau$. Then the estimation of $\exp(-|Q_k-Q^*|)$ becomes 
$$\exp \left(-\frac{\gamma\left[P^{\pi^{w_{k-1}}} \Delta_{k-1}\right](s, a)}{\tau}\right)
$$
\paragraph{Truncated TCE} TCE may suffer from a big deviation when $h^\pi_\tau(s,a)$ is too large or too small. To tackle this issue and improve the stability of the prioritization, we clip the output of TCE into $[b_1, b_2]$, where $b_1$ and $b_2$ are regarded as hyperparameters.

\paragraph{Baselines}
For the ReMERN and ReMERT algorithms in continuous action spaces with sensory observation, we alter the re-weighting strategy to $\frac{d^{\pi_k}(s,a)}{\mu(s,a)}$ and TCE approximation based on the source code provided by DisCor\footnote{https://github.com/ku2482/discor.pytorch}. For the algorithms in discrete action spaces with pixel observation, we employ the baseline Tianshou\footnote{https://github.com/thu-ml/tianshou}~\cite{tianshou} and add corresponding components.

\subsubsection{Hyperparameter Details}
The hyperparameters of our ReMERN and ReMERT algorithms include network architectures, learning rates, temperatures in on-policy reweight and DisCor, and the lower and upper bound in TCE algorithm. They are specified as follows:
\begin{itemize}[leftmargin=*]
    \item \textbf{Network architectures}~~We use standard Q and policy network in MuJoCo benchmark with hidden network sizes [256, 256]. In Meta-World we add an extra layer and the hidden network sizes are [256, 256, 256]. The networks computing $\Delta$ and $\kappa$  have one extra layer than the corresponding Q and policy network.
    \item \textbf{Learning rates}~~The learning rate for continuous control tasks, including Meta-World, MuJoCo and DMC, is set to be 3e-4 for Q and policy networks alike. For Atari games, the learning rate is set to be 1e-4 and fixed across all environments.
    
    \item \textbf{Temperatures}~~The temperature for weights related with $\frac{d^{\pi_k}(s,a)}{\mu(s,a)}$ is 7.5 and fixed across different environments. Also, DisCor has a temperature hyperparameter related to the output normalization of the error network. We keep it unchanged in the Meta-World and DMC benchmark, and divide it by 20 in MuJoCo environments to make it compatible with on-policy prioritization weights.
    \item \textbf{Bounds in TCE}~~We select time-adaptive lower and upper bounds for TCE. The lower bound rises from 0.4 when training begins to 0.9 when it ends, and the upper bound drops from 1.6 to 1.1 accordingly. The bounds are fixed across different environments.
    \item \textbf{Random Seeds}~~In MuJoCo, Meta-World and DMC benchmarks, we run each experiment with four random seeds. The results are plotted with the mean of the four experiments. In Atari games, we run experiments with three random seeds and select the one with max return.
\end{itemize}

\end{document}